\documentclass[aoas]{imsart}

\RequirePackage{amsthm,amsmath,amsfonts,amssymb}
\RequirePackage[authoryear]{natbib}
\RequirePackage[colorlinks,citecolor=blue,urlcolor=blue]{hyperref}
\RequirePackage{graphicx}

\newcommand{\figref}[1]{Figure~\ref{#1}}

\newcommand{\tabref}[1]{Table~\ref{#1}}

\newcommand{\secref}[1]{Section~\ref{#1}}

\newcommand{\appref}[1]{Section~\ref{#1}}

\newcommand{\algoref}[1]{Algorithm~\ref{#1}}
\newcommand{\eqnref}[1]{Equation~\ref{#1}}
\renewcommand{\eqref}[1]{\eqnref{#1}}

\newcommand{\cifar}[0]{CIFAR-10\xspace}
\newcommand{\cifarb}[0]{CIFAR-100\xspace}
\newcommand{\cinic}[0]{CINIC-10\xspace}
\newcommand{\wrn}[0]{Wide-ResNet\xspace}

\newcommand{\contract}[0]{contraction ratio\xspace}
\newcommand{\varia}[0]{variation coefficient\xspace}

\newcommand{\mia}[0]{MIA\xspace}
\newcommand{\mias}[0]{MIAs\xspace}
\newcommand{\lira}[0]{LiRA\xspace}
\newcommand{\rmia}[0]{RMIA\xspace}
\newcommand{\plmia}[0]{PL-MIA\xspace}
\newcommand{\plmiap}{PL-MIA$^{+}$\xspace}

\usepackage{booktabs}
\usepackage[table]{xcolor}
\usepackage{algorithm}
\usepackage{algorithmic}
\usepackage{pifont}
\usepackage{xspace}
\usepackage{multirow}
\usepackage{adjustbox}
\usepackage[percent]{overpic}
\usepackage{subcaption}
\usepackage{makecell}

\startlocaldefs
\theoremstyle{plain}
\newtheorem{theorem}{Theorem}[section]

\newtheorem{lemma}[theorem]{Lemma}

\theoremstyle{definition}
\newtheorem{definition}[theorem]{Definition}
\newtheorem{assumption}[theorem]{Assumption}

\endlocaldefs

\begin{document}

\begin{frontmatter}
\title{Membership Inference via Pairwise Likelihood Ratios}
\runtitle{Pairwise Likelihood MIA}

\begin{aug}
\author[A]{\fnms{Shengjie}~\snm{Niu}\ead[label=e1]{shengjie.niu@connect.polyu.hk}}
\author[B]{\fnms{Zebin}~\snm{Yun}\ead[label=e2]{zebinyun@mail.tau.ac.il}}
\author[A]{\fnms{Yeheng}~\snm{Ge}\ead[label=e3]{yeheng.ge@polyu.edu.hk}}
\author[A]{\fnms{Jian}~\snm{Huang}\ead[label=e4]{j.huang@polyu.edu.hk}}
\address[A]{The Hong Kong Polytechnic University, Hong Kong, China\printead[presep={,\ }]{e1}\printead[presep={,\ }]{e3}\printead[presep={,\ }]{e4}}

\address[B]{Tel Aviv University, Israel \printead[presep={,\ }]{e2}}
\end{aug}

\begin{abstract}
Membership inference attacks (MIAs) are the standard tool for auditing the privacy risks of machine learning models. Given a query point, an MIA aims to determine whether that point was used to train the target model. In practice, such inference must rely on the statistical signals exposed by the model's outputs, such as confidence scores, logits, and intermediate feature representations. However, existing methods often fail to efficiently summarize and combine these statistical signals. To address this limitation, we propose Pairwise Likelihood MIA (PL-MIA), a unified method that combines a Gaussian likelihood-ratio (GLR) statistic with population calibration and the Cauchy combination test. We characterize theoretically how the GLR retains variance-contraction signals and establish conditions under which population calibration and Cauchy combination improve attack power. We obtain $p$-values from pairwise comparisons between the query point and reference points not used for training, and aggregate these continuous signals using the Cauchy combination test. This preserves the evidence strength that is discarded when each pairwise comparison is reduced to a binary vote. Extensive experiments demonstrate that PL-MIA outperforms strong baselines, improving the true positive rate (TPR) by over 25\% in the critical low-false-positive regime, corroborating our theoretical findings. These results demonstrate how statistical principles can turn noisy model outputs into more powerful, calibrated, and reproducible evidence for membership privacy auditing.
\end{abstract}


\end{frontmatter}

\section{Introduction}\label{sec:intro}

Machine-learning models trained on individual-level data may reveal whether a particular record was included in their training sets.
Such membership information can itself be sensitive. For example, identifying an individual as a member of a biomedical, genomic, or financial dataset may disclose private information about that individual or their participation in a sensitive study~\citep{homer2008resolving,backes2016membership,hernandez2024financial}.
Membership inference attacks (MIAs)~\citep{mia} provide a data-driven tool for auditing this risk: given black-box access to a trained model and a query point, an auditor attempts to determine whether that point was used for training.
A central practical question is therefore whether an auditor can make reliable membership inferences while keeping false accusations of non-members rare.
This requirement makes the low false-positive-rate (low-FPR) regime particularly important.
An attack with satisfactory average discrimination ability may still be unsuitable for privacy auditing if most of its positive findings are unsupported.
Therefore, effective MIAs must achieve high power in this low-FPR regime.
However, this objective is challenging in practical black-box settings, where membership can only be inferred from model outputs that reflect not only membership status but also intrinsic properties of the query point.
For instance, an intrinsically easy point may receive high confidence even when it was not used for training, whereas a difficult training point may receive relatively low confidence.
A credible audit must therefore separate membership-related signal from query-level heterogeneity.

Membership inference is fundamentally a statistical decision problem: we must effectively extract membership information from limited and heterogeneous model outputs, reliably aggregate diverse statistical evidence, and translate it into sound membership decisions. This objective motivates the development of principled statistical methodologies, each addressing a distinct part of the pipeline.
First, likelihood theory guides the construction of efficient test statistics that extract membership signal from noisy model outputs, while population calibration separates the membership-related component of the signal from query-level heterogeneity arising from query difficulty. Second, combination tests aggregate multiple, potentially dependent statistical signals into a single membership score, maximizing the information retained in the final statistic. Third, statistical uncertainty quantification converts the resulting scores into decisions with controlled FPRs, a particularly critical capability in the extreme low-FPR regime.
While these three components form a coherent framework, the signal aggregation via combination tests and decision-making with rigorous FPR control constitute the primary methodological contributions of our article.
In particular, existing MIAs typically determine decision thresholds on additional held-out data~\citep{lira}, which both impose extra data requirements and introduce variability into threshold estimation. We therefore derive decision thresholds analytically, without held-out data, while keeping the realized FPR at or below its nominal level across heterogeneous queries.

This statistical perspective also reveals the limitations of existing strong MIAs. Recent methods increasingly formulate membership inference as a hypothesis-testing problem~\citep{lira}.
For each query point, reference models are used to estimate the distributions of model outputs when the point is included in or excluded from the training set, referred to as the IN and OUT distributions, respectively~\citep{sablayrolles2019white,lira,watson2021importance,rmia}.
LiRA~\citep{lira} models these distributions separately and uses a Gaussian likelihood ratio (GLR) as a test statistic to extract a fine-grained membership signal. This statistic exploits both the shift in predictive confidence and the empirical contraction of predictive variance observed for training members. However, directly applying a global threshold to this statistic does not account for query-level heterogeneity.
RMIA~\citep{rmia} addresses query-level heterogeneity through population calibration by comparing the query point with randomly sampled population points, but its test statistic does not explicitly exploit the variance-contraction signal.
Moreover, its population calibration reduces each pairwise comparison to a binary indicator, treating a marginal comparison and an overwhelmingly strong comparison as equally informative. Thus, existing methods do not simultaneously extract the available membership signal, adjust for query-level heterogeneity, and preserve the strength of pairwise evidence.

Motivated by this gap, we propose a unified framework that decomposes MIAs into two components: (1) a pointwise statistic that extracts membership signal for each query point, and (2) a population calibration strategy that normalizes the query point's statistic using population data to reduce intrinsic query-level heterogeneity.
Building on this unified framework, we propose Pairwise Likelihood MIA (\textbf{\plmia}).
\emph{First}, we adopt the GLR as the pointwise statistic.
We theoretically show that the GLR explicitly captures the \emph{variance-contraction} signal and therefore yields greater separation between the member and non-member statistic distributions than the BLR.
\emph{Second}, we design an improved population calibration strategy based on \emph{continuous} pairwise evidence.
For each pairwise comparison between the query point and a randomly sampled population point, we convert the pairwise difference into a continuous $p$-value and then aggregate multiple $p$-values using the Cauchy combination rule~\citep{cct}.
In contrast to the threshold-based binary evidence~\citep{rmia}, continuous $p$-values preserve the strength of each pairwise comparison.
Meanwhile, the Cauchy combination rule improves sensitivity to subtle membership signals~\citep{liu2019acat} and enables an analytic decision threshold at a target FPR.
\emph{Theoretically}, we show that population calibration centers non-member scores around their corresponding neutral baselines, providing a common scale for thresholding membership evidence.
We further derive a unified expression for attack power across strong attacks, showing that \plmia improves over \lira and \rmia by simultaneously exploiting the variance-contraction signal and reducing query-level heterogeneity.

The contributions of this study are summarized as follows:
\textbf{(1)} We propose a unified framework that characterizes MIAs as the composition of a \emph{pointwise statistic} and a \emph{population calibration} strategy.
\textbf{(2)} We propose Pairwise Likelihood MIA (PL-MIA), which integrates the Gaussian likelihood ratio, population calibration, and Cauchy combination into a statistically principled attack.
\textbf{(3)} We theoretically characterize the advantage of \plmia over strong baselines by simultaneously exploiting variance contraction and reducing query-level heterogeneity.
\textbf{(4)} Extensive experiments demonstrate that \plmia consistently improves attack performance across datasets and achieves substantial gains in the critical low-FPR regime.

\section{Preliminary}
\label{sec:preliminary}

Throughout the paper, we denote the target model, parameterized by $\theta_t$, as $f_{\theta_t}$.
Membership Inference Attacks~(\mias)~\citep{mia} aim to determine whether a query point was used to train the target model.
We formalize MIA through a security game between a challenger (model owner) and an adversary (privacy auditor)~\citep{yeom2018privacy, jayaraman2020revisiting,lira}.

\begin{definition}[Membership Inference Game]
\label{def:mi_game}
Let $\pi$ be the underlying data distribution, $\mathcal{T}$ the challenger's training algorithm, and $\mathcal{A}$ the adversary's membership inference attack. The \mia game proceeds as follows:
\begin{enumerate}
    \item The challenger samples a training dataset $\mathcal{D} \sim \pi^n$ and trains the target model $f_{\theta_t} \leftarrow \mathcal{T}(\mathcal{D})$.
    \item The challenger samples a membership bit $b \sim \mathrm{Bernoulli}(1/2)$. If $b=0$, it samples a fresh query point $q \sim \pi$ subject to $q \notin \mathcal{D}$; if $b=1$, it samples $q$ uniformly at random from $\mathcal{D}$.
    \item The challenger provides the query point $q$ and access to the target model $f_{\theta_t}$ to the adversary, according to the specified threat model.
    \item The adversary outputs a membership prediction $\hat{b} \leftarrow \mathcal{A}(f_{\theta_t},q)$ and wins the game if
    $\hat{b}=b$.
\end{enumerate}
\end{definition}
The membership bit $b$ determines the membership status of the query point: when $b=1$, $q$ is sampled from the target model's training set $\mathcal{D}$ and is therefore a member; when $b=0$, $q$ is sampled from the underlying distribution $\pi$ but excluded from $\mathcal{D}$ and is therefore a non-member. The choice $b\sim\mathrm{Bernoulli}(1/2)$ is a standard convention that assigns equal prior probability to the member and non-member cases~\citep{jayaraman2020revisiting}.

\paragraph*{Threat Model}
A threat model specifies which information is available to the adversary for constructing statistical evidence of membership.
Following prior work~\citep{watson2021importance,emia,bertran2023scalable,rmia}, we consider a worst-case black-box threat model.
The adversary has access to the underlying data distribution $\pi$ and can replicate the target model's architecture and training algorithm~\citep{zhu2025impact}, as in settings where both the adversary and challenger use the same machine learning-as-a-service provider~\citep{salem2018ml}.
However, the adversary has no access to the target model’s training set or trained parameters and is limited to querying the target model and observing its outputs.

\paragraph*{Membership Score}
Inputting a query point $q$ into a target model,
a MIA assigns
a membership score $\mathrm{Score}(q)$,
where a larger score indicates stronger evidence of membership.
The adversary converts this score into a membership prediction by thresholding it at a decision threshold $\tau_\alpha$~\citep{rmia}:
\begin{equation}
    \label{eq:member_score}
    \hat{b} = \mathcal{A}(f_{\theta_t},q) = \mathbb{I}(\mathrm{Score}(q) \geq \tau_\alpha),
\end{equation}
where $\mathbb{I}(\cdot)$ is the indicator function and $\tau_\alpha$ is chosen to target a false positive rate (FPR) of $\alpha$. 
Since practical audits often require near-zero false accusations, we prioritize attacks that achieve high TPR under stringent FPR constraints.
For existing MIAs, $\tau_\alpha$ is typically estimated using held-out data or derived analytically~\citep{lira}.
In PL-MIA, pairwise $p$-values are aggregated using the Cauchy combination rule, yielding an analytic decision threshold for nominal FPR control without requiring additional held-out data.

\paragraph*{Statistical Hypothesis}
We formalize the \mia test by defining two distributions over model parameters: $\Theta^{(q)}_{\mathrm{IN}}$ and $\Theta^{(q)}_{\mathrm{OUT}}$.
For a given query point $q$, $\Theta^{(q)}_{\mathrm{IN}}$ denotes the distribution of parameters of reference models trained on datasets that \emph{include} $q$, whereas $\Theta^{(q)}_{\mathrm{OUT}}$ denotes the analogous distribution when $q$ is \emph{excluded}.
The null and alternative hypotheses are
\begin{equation}
    H_0: \theta_t \sim \Theta^{(q)}_{\mathrm{OUT}}
    \quad \mathrm{versus} \quad
    H_1: \theta_t \sim \Theta^{(q)}_{\mathrm{IN}}.
\end{equation}
Formally, a model parameter $\theta_t$ is drawn from the IN distribution if trained on a dataset containing the query point $q$, and from the OUT distribution otherwise.
By the Neyman-Pearson lemma~\citep{neyman1933ix}, the most powerful test can be constructed using the likelihood ratio (LR):
\begin{equation}
    \label{eq:lrt}
    \Lambda^*(\theta_t,q) = \frac{\mathrm{pdf}(\theta_t|\Theta^{(q)}_{\mathrm{IN}})}{\mathrm{pdf}(\theta_t|\Theta^{(q)}_{\mathrm{OUT}})},
\end{equation}
where $\mathrm{pdf}(\theta_t|\Theta^{(q)}_{\mathrm{b}})$ denotes the density of observing $\theta_t$ under hypothesis $b \in \{ \mathrm{IN}, \mathrm{OUT} \}$.
When $q$ is a training member, $\theta_t$ is typically more likely under $\Theta^{(q)}_{\mathrm{IN}}$ than under $\Theta^{(q)}_{\mathrm{OUT}}$, leading to a larger value of $\Lambda^*(\theta_t,q)$.

Despite this optimality, directly modeling $\Theta^{(q)}_{\mathrm{IN}}$ and $\Theta^{(q)}_{\mathrm{OUT}}$ is generally intractable in the black-box setting, where the target model parameters $\theta_t$ are not observable.
Following~\citet{lira}, we therefore construct an observable statistic from the target model output for a query point $q=(x_q,y_q)$:
$$\mathrm{conf}^{(q)}_t = \phi(f_{\theta_t}(x_q)[y_q]),$$
where $f_{\theta_t}(x_q)[y_q]$, referred to as the true label confidence (TLC)\footnote{For a classification model, the input $x_q$ represents the observed features, and the model outputs a probability vector over possible classes.
In this work, we consider the probability assigned to the true class $y_q$ as the observable signal for membership inference under the black-box setting.}, denotes the probability assigned by the target model to the true label $y_q$ for the input $x_q$, and $\phi(p) = \log \frac{p}{1-p}$ is the logit transform.
Motivated by empirical evidence that the logit-scaled TLC values are well approximated by Gaussian distributions~\citep{lira}, we introduce the following assumption to make the hypothesis-testing problem analytically tractable.
We present normality diagnostics and evaluate robustness to violations of this Gaussian assumption in Section~S2 of the Supplementary Material.

\begin{assumption}[Gaussian Distribution]
\label{ass:gaussian_model}
For any query point $q$, the logit-transformed confidence scores follow Gaussian distributions conditioned on membership status:
\begin{equation}
\label{eq:gaussian_model}
\mathrm{conf}^{(q)}_{\mathrm{IN}} \sim \mathcal{N}\left(\mu_{\mathrm{in}}^{(q)}, (\sigma_{\mathrm{in}}^{(q)})^2\right), \quad
\mathrm{conf}^{(q)}_{\mathrm{OUT}} \sim \mathcal{N}\left(\mu_{\mathrm{out}}^{(q)}, (\sigma_{\mathrm{out}}^{(q)})^2\right),
\end{equation}
where $\mathrm{conf}^{(q)}_{\mathrm{IN}}$ and $\mathrm{conf}^{(q)}_{\mathrm{OUT}}$ denote the induced distributions of $\phi(f_{\theta}(x_q)[y_q])$ when $\theta \sim \Theta^{(q)}_{\mathrm{IN}}$ and  $\theta \sim \Theta^{(q)}_{\mathrm{OUT}}$, respectively.
\end{assumption}

This assumption shifts the distinction between the IN and OUT model populations onto their output distributions for the query point $q$, which are characterized by three quantities.
(1) the \textit{normalized mean shift} $\eta^{(q)} := (\mu_{\mathrm{in}}^{(q)} - \mu_{\mathrm{out}}^{(q)}) / \sigma_{\mathrm{out}}^{(q)} > 0$, representing the normalized increase in predictive confidence of members;
(2) the \textit{\contract} $\rho^{(q)}:= \sigma_{\mathrm{in}}^{(q)} / \sigma_{\mathrm{out}}^{(q)} \leq 1$, capturing the reduced predictive uncertainty of members; and
(3) the \textit{\varia} $\xi^{(q)} := [1-\operatorname{sigmoid}(\mu_{\mathrm{out}}^{(q)})]
\sigma_{\mathrm{out}}^{(q)}$, representing the leading-order coefficient
of variation of the raw TLC under $H_0$.
We restrict attention to tests whose evidence depends on $\theta_t$ only through the statistic $\phi(f_{\theta_t}(x_q)[y_q])$.

Within this class of tests, membership testing based on $\theta_t$ reduces to testing the induced observable statistic $\mathrm{conf}^{(q)}_t$.
Therefore, the likelihood ratio in \eqref{eq:lrt} becomes
\begin{equation}
    \label{eq:lrt_conf}
    \Lambda(\theta_t,q) = \frac{
    \mathrm{pdf}(\mathrm{conf}^{(q)}_t|\Theta^{(q)}_{\mathrm{IN}})
    }{
    \mathrm{pdf}(\mathrm{conf}^{(q)}_t|\Theta^{(q)}_{\mathrm{OUT}})}.
\end{equation}
The LR statistic is based on the empirical observation that models trained with a query point $q$ tend to assign higher confidence to its true label, whereas models trained without $q$ tend to assign lower confidence to the same point.
The statistic $\mathrm{conf}^{(q)}_t$ is calculated from the model output and depends on the underlying parameter $\theta_t$. It serves as an observable surrogate for parameter-level membership evidence in the black-box setting.

\section{Methodology}
\label{sec:method}

\subsection{A Unified Framework of MIA}
We present a unified framework that characterizes MIAs as the composition of two components: a \emph{pointwise statistic} and a \emph{population calibration strategy}. This decomposition clarifies differences among prior attacks and guides our attack design.
We summarize recent strong attacks in \tabref{tab:improvement_sources}.

\begin{table*}[th]
\centering
\caption{Comparison of representative MIA methods under the unified framework.
We compare Attack-R and Attack-P~\citep{emia}, \lira~\citep{lira}, \rmia~\citep{rmia}, and \plmia in terms of their pointwise statistics and population calibration strategies; for attacks with population calibration, the latter is further decomposed into evidence generation and evidence aggregation.}
\label{tab:improvement_sources}
\begin{adjustbox}{width=\linewidth}
\begin{tabular}{@{}lccccc@{}}
\toprule
\textbf{Attacks} &
\makecell{\textbf{Attack-R}} &
\makecell{\textbf{Attack-P}} &
\makecell{\textbf{\lira}} &
\makecell{\textbf{\rmia}} &
\makecell{\textbf{\plmia (OURS)}} \\
\midrule
\textbf{Pointwise Statistic} & LOSS & LOSS & Gaussian LR & Bayes LR & Gaussian LR \\
\hspace{1.5mm} \textit{Estimates IN Dist.} & \ding{56} & \ding{56} & \ding{52} & \ding{52} & \ding{52} \\
\hspace{1.5mm} \textit{Exploits Contraction ($\rho$)} & \ding{56} & \ding{56} & \ding{52} & \ding{56} & \ding{52} \\ \midrule
\textbf{Population Calibration} & \ding{56} & \ding{52} & \ding{56} & \ding{52} & \ding{52} \\
\hspace{1.5mm} \textit{Evidence Generation} & - & Binary indicator  & - & Binary indicator & Continuous $p$-value \\
\hspace{1.5mm} \textit{Evidence Aggregation} & - & Mean & - & Mean & Cauchy combination \\
\bottomrule
\end{tabular}
\end{adjustbox}
\vspace{-1em}
\end{table*}

\paragraph*{Pointwise Statistic}

In our framework, each \mia first constructs a scalar pointwise statistic $\varphi^{(d)}$ that extracts the individual membership signal for a query point $d$.
A larger value of $\varphi^{(d)}$ indicates stronger evidence for membership.
The pointwise statistic serves as the fundamental building block of an MIA, and different attacks correspond to different choices of $\varphi^{(d)}$.
Recent strong attacks construct $\varphi^{(d)}$ through hypothesis testing, using multiple reference models to extract finer-grained membership signals~\citep{lira,watson2021importance,rmia}.
As discussed in \secref{sec:preliminary}, the most powerful tests are based on likelihood ratios (LRs); accordingly, we focus on attacks that use LR estimators as the pointwise statistic~\citep{lira,rmia}:
\begin{itemize}
    \item \textbf{LiRA}~\citep{lira} uses the Gaussian LR:
    $\varphi_{\mathrm{GLR}}^{(d)}  =
    \log \frac{\mathrm{pdf}(\mathrm{conf}^{(d)}_t|\mathcal{N}(\mu_{\mathrm{in}}^{(d)},(\sigma_{\mathrm{in}}^{(d)})^2))}{\mathrm{pdf}(\mathrm{conf}^{(d)}_t|\mathcal{N}(\mu_{\mathrm{out}}^{(d)},(\sigma_{\mathrm{out}}^{(d)})^2))}$. It estimates the LR by fitting Gaussian models to the $\mathrm{conf}^{(d)}_{\mathrm{IN}}$ and $\mathrm{conf}^{(d)}_{\mathrm{OUT}}$ distributions.
    \item \textbf{RMIA}~\citep{rmia} employs the Bayes LR: $\varphi_{\mathrm{BLR}}^{(d)} := \log (\frac{f_{\theta_t}(x_d)[y_d]}{\bar c^{(d)}})$, where $\bar c^{(d)}:=\mathbb{E}[f_{\theta}(x_d)[y_d]]$ is the average raw TLC over all reference models.
\end{itemize}

\paragraph*{Population Calibration Strategy}
This strategy compares the statistics of the query point with those of population points to neutralize intrinsic query-level variability.
It transforms the pointwise statistic $\varphi^{(d)}$ into the membership score used in \eqref{eq:member_score} through two steps: (i) \emph{Evidence Generation}, which forms membership evidence from each pairwise comparison between the query point and a randomly sampled population point, and (ii) \emph{Evidence Aggregation}, which combines the generated evidence from multiple pairwise comparisons into a single membership score.
\begin{itemize}
    \item \textbf{Without  Calibration} \citep{lira}: The membership score is set to the statistic itself, $\mathrm{Score_{LiRA}}(q):= \varphi^{(q)}_{\mathrm{GLR}}$, and is compared to a global decision threshold to infer membership. As noted by \citeauthor{rmia}, this approach does not calibrate against population data.
    \item \textbf{With Calibration} \citep{rmia}: The query point's statistic $\varphi^{(q)}_{\mathrm{BLR}}$ is calibrated against the statistics $\{\varphi^{(z)}_{\mathrm{BLR}}\}_{z \in Z}$ of a set of population points $Z$,
    \begin{equation}
        \label{eq:rmia_membership_score}
        \mathrm{Score}_{\mathrm{RMIA}}(q) := \Pr_{z \in Z} [\varphi^{(q)}_{\mathrm{BLR}} - \varphi^{(z)}_{\mathrm{BLR}} > \gamma].
    \end{equation}
Here, population calibration tests whether the query point $q$ $\gamma$-dominates a randomly sampled population point $z$. Each pairwise comparison yields one binary indicator, and evidence aggregation takes the mean of these bits to produce the final membership score.
\end{itemize}

\subsection{Designing PL-MIA}
\label{sec:optimal_llr}

Under this unified framework, we propose Pairwise Likelihood MIA (\plmia) by integrating statistically principled designs for both components.
It adopts the GLR pointwise statistic and uses improved population calibration with a continuous evidence generation and aggregation mechanism.
We provide the pseudocode of \plmia in \algoref{alg:plmia}.

\begin{algorithm}[!htbp]
\caption{\textbf{MIA Score Computation with PL-MIA.} Inputs: reference models $\Theta$ with training-membership records, target model $\theta_t$, query point $q=(x_q,y_q)$, population dataset $\mathcal{D}_{\mathrm{pop}}$, attack mode $m\in\{\mathrm{online},\mathrm{offline}\}$, and optional offline hyperparameters $\eta$, $\rho$.}
\label{alg:plmia}
\begin{algorithmic}[1]
    \STATE Randomly choose a hold-out subset $Z \subset \mathcal{D}_{\mathrm{pop}}\setminus\{q\}$; \quad $C \gets 0$

    \STATE \textbf{// Step 1: Estimate IN/OUT distributions for all relevant data points}
    \FOR {each data point $d \in \{q\} \cup Z$}
        \STATE $\mathrm{confs}_{\mathrm{IN}}^{(d)} \gets \emptyset$; \quad $\mathrm{confs}_{\mathrm{OUT}}^{(d)} \gets \emptyset$
        \FOR {each reference model $\theta' \in \Theta$}
            \STATE $\mathrm{conf} \gets \phi\big( f_{\theta'}(x_d)[y_d] \big)$ \hfill {\color{gray}\# Query reference model}
            \IF {$d \in$ training set of $\theta'$}
                \STATE $\mathrm{confs}_{\mathrm{IN}}^{(d)} \gets \mathrm{confs}_{\mathrm{IN}}^{(d)} \cup \{\mathrm{conf}\}$
            \ELSE
                \STATE $\mathrm{confs}_{\mathrm{OUT}}^{(d)} \gets \mathrm{confs}_{\mathrm{OUT}}^{(d)} \cup \{\mathrm{conf}\}$
            \ENDIF
        \ENDFOR
        \STATE $\mu_{\mathrm{OUT}}^{(d)} \gets \mathrm{mean}(\mathrm{confs}_{\mathrm{OUT}}^{(d)})$; \quad $\sigma_{\mathrm{OUT}}^{(d)} \gets \mathrm{std}(\mathrm{confs}_{\mathrm{OUT}}^{(d)})$
        \IF {$m=\mathrm{online}$}
            \STATE $\mu_{\mathrm{IN}}^{(d)} \gets \mathrm{mean}(\mathrm{confs}_{\mathrm{IN}}^{(d)})$; \quad $\sigma_{\mathrm{IN}}^{(d)} \gets \mathrm{std}(\mathrm{confs}_{\mathrm{IN}}^{(d)})$
        \ELSE
            \STATE $\sigma_{\mathrm{IN}}^{(d)} \gets \rho\cdot\sigma_{\mathrm{OUT}}^{(d)}$; \quad $\mu_{\mathrm{IN}}^{(d)} \gets \mu_{\mathrm{OUT}}^{(d)}+\eta\cdot\sigma_{\mathrm{OUT}}^{(d)}$
        \ENDIF
    \ENDFOR

    \STATE \textbf{// Step 2: Query target model and compute the GLR pointwise statistic}
    \FOR {each data point $d \in \{q\} \cup Z$}
        \STATE $\mathrm{conf}_{t}^{(d)} \gets \phi\big( f_{\theta_t}(x_d)[y_d] \big)$ \hfill {\color{gray}\# Query target model}
        \STATE $\varphi^{(d)}_{\mathrm{PL}} \gets \log \mathrm{pdf}\big(\mathrm{conf}_{t}^{(d)} \mid \mathcal{N}(\mu_{\mathrm{IN}}^{(d)}, (\sigma_{\mathrm{IN}}^{(d)})^2)\big) - \log \mathrm{pdf}\big(\mathrm{conf}_{t}^{(d)} \mid \mathcal{N}(\mu_{\mathrm{OUT}}^{(d)}, (\sigma_{\mathrm{OUT}}^{(d)})^2)\big)$
    \ENDFOR
    \STATE $\widehat{\sigma}_{\Delta} \gets \sqrt{2}\cdot \mathrm{std}\big(\{\varphi^{(z)}_{\mathrm{PL}}:z\in Z\}\big)$ \hfill {\color{gray}\# Null scale for pairwise $p$-values}

    \STATE \textbf{// Step 3: Pairwise comparison and Cauchy aggregation}
    \FOR {each sample $z \in Z$}
        \STATE $\Delta(q,z) \gets \varphi^{(q)}_{\mathrm{PL}} - \varphi^{(z)}_{\mathrm{PL}}$
        \STATE $p(q,z) \gets \Phi\big(-\Delta(q,z)/\widehat{\sigma}_{\Delta}\big)$
        \STATE $C \gets C + \tan\big((0.5-p(q,z))\cdot\pi\big)$
    \ENDFOR

    \STATE \textbf{return} $\mathrm{Score}_{\mathrm{PL}}(q;\theta_t) \gets C/|Z|$
\end{algorithmic}
\end{algorithm}

\paragraph*{Pairwise Difference as Membership Evidence}
The core of population calibration lies in comparing the query point $q$ against a set of population points $Z$ to neutralize the intrinsic query-level heterogeneity.
We start by defining the \emph{pairwise difference} as the difference between the pointwise statistic of $q$ and a randomly sampled population point $z \in Z$:
\begin{equation}
    \Delta(q, z) = \varphi^{(q)}_{\mathrm{PL}} - \varphi^{(z)}_{\mathrm{PL}},
\end{equation}
where we adopt the GLR as the pointwise statistic, $\varphi_{\mathrm{PL}} = \varphi_{\mathrm{GLR}}$.
This difference $\Delta(q, z)$ serves as the fundamental unit of membership evidence for each pairwise comparison. Under $H_1$, where $q$ is a member, its statistic tends to exceed that of non-member points, resulting in positive pairwise differences on average. Under $H_0$, where both $q$ and $z$ are non-members, $\Delta(q, z)$ is approximately symmetric around zero.

\paragraph*{From Binary Voting to Continuous $p$-values}
A standard population calibration strategy to aggregate these pairwise differences into a membership score is to count the proportion of population points that the query point $q$ can $\gamma$-dominate~\citep{rmia}:
\begin{equation}
    \label{eq:binary_calib}
    \mathrm{Score}_{\mathrm{Binary}}(q) := \Pr_{z \in Z} (\Delta(q, z) > \gamma).
\end{equation}
While effective, this binary voting mechanism has a limitation: it discretizes the pairwise evidence using a hard threshold $\gamma$.
Consequently, a population point that marginally exceeds the threshold contributes the same evidence as one that exceeds it substantially.
This binarization discards the \emph{magnitude} of $\Delta(q, z)$, which reflects the strength of the membership evidence.

To preserve this information, we consider the $p$-value as a \emph{continuous} measure of statistical significance.
We emphasize that it retains the strength information of each pairwise comparison compared with the binary voting.
Because under $H_1$ the relevant departure is directional ($\Delta(q,z) > 0$), we employ the \emph{upper-tail} test to compute the \emph{one-sided} $p$-value for each pairwise comparison.
Specifically, let $\widetilde{\Delta}$ be the random variable representing the pairwise difference between the statistics of two non-members.
Under $H_0$, both the query point and the population point are non-members, and $\widetilde{\Delta}$ is centered around zero. We approximate its null distribution by a centered Gaussian, which is reasonable since $\widetilde\Delta$ is the difference of two log-likelihood-ratio statistics~\citep{van2000asymptotic}.
For the observed difference $\Delta_{\mathrm{obs}}=\Delta(q,z)$, the one-sided $p$-value is computed as
\begin{equation}
\begin{aligned}
    \label{eq:p_value}
    p(q, z)  = \Pr(\widetilde{\Delta} \ge \Delta_{\mathrm{obs}} \mid H_0)
     \approx \Phi\left(-\frac{\Delta(q, z)}{\hat{\sigma}_\Delta}\right),
\end{aligned}
\end{equation}
where $\hat{\sigma}_\Delta$ is the estimated standard deviation of the pairwise differences under $H_0$, and $\Phi(\cdot)$ is the standard normal cumulative distribution function (CDF).
This transformation maps the raw difference $\Delta(q, z)$ to a probability scale $p(q, z) \in [0, 1]$, where smaller values indicate stronger membership evidence.
The binary score in \eqnref{eq:binary_calib} can be viewed as a coarse tail probability over population draws $z\in Z$, whereas the $p$-value in \eqref{eq:p_value} retains how strongly each pairwise comparison departs from the non-member null.
Thus, this continuous $p$-value serves as a smoothed surrogate for the binary indicator, preserving the evidence strength of each pairwise comparison.
We empirically evaluate the validity of these pairwise $p$-values in \secref{sec:exp_property}.

\paragraph*{Evidence Aggregation via Cauchy Combination}
For a given query point $q$ and a set of continuous $p$-values $\{p(q, z)\}_{z \in Z}$, we aggregate them into a membership score using the Cauchy Combination rule~\citep{liu2019acat,cct}:
\begin{equation}
    \label{eq:cauchy_cali}
    \mathrm{Score_{PL}}(q) = \frac{1}{|Z|} \sum_{z \in Z} \tan\big((0.5 - p(q, z)) \cdot \pi\big),
\end{equation}
where the Cauchy transformation $T(p) =\tan \big((0.5 - p) \cdot \pi \big)$ maps each $p$-value to $\mathbb{R}$ via the inverse CDF of the standard Cauchy distribution.
This transformation maps near-zero $p$-values, indicating strong membership evidence, to large positive values, while $p$-values near $0.5$, representing neutral evidence, are mapped to values near zero.
This aggregation is particularly effective when membership evidence is concentrated in a subset of pairwise comparisons: strong evidence from a small number of comparisons can dominate the aggregated score, thereby reducing dilution from numerous weak or uninformative comparisons.

\subsection{Why PL-MIA Is More Powerful}
We argue that \plmia achieves stronger attack power by combining a more informative pointwise statistic with effective population calibration and practical evidence aggregation.

\paragraph*{Statistical Efficiency of GLR}
The GLR provides a more informative pointwise statistic by exploiting both the IN and OUT distributions.
Under Assumption~\ref{ass:gaussian_model}, GLR jointly captures two sources of membership signal: the \emph{normalized mean shift} $\eta$ and the \emph{variance-contraction ratio} $\rho$.
These signals reflect the tendency of overfitting to increase confidence in training members while reducing their predictive uncertainty relative to non-members.
In contrast, BLR~\citep{rmia} primarily relies on the normalized mean shift and does not explicitly exploit the variance-contraction signal.
Our theoretical analysis (Lemma~\ref{lem:logbf} and Lemma~\ref{lem:rmia}) shows that GLR provides stronger statistical separation than BLR when variance contraction or nuisance variation is present.
This stronger pointwise separation produces more informative pairwise evidence for the subsequent population calibration.

\paragraph*{Population Calibration for Neutralizing Intrinsic Heterogeneity}
A strong pointwise statistic alone is insufficient because model outputs also reflect intrinsic query-level heterogeneity unrelated to membership~\citep{watson2021importance}.
PL-MIA addresses this limitation through population calibration, which compares the statistic of the query point with that of the population points and transforms the raw difference into relative membership evidence.
This calibration factors out query-level heterogeneity, allowing the final membership score to focus on the signal attributable to membership status.
Our theoretical analysis (Lemma~\ref{lem:score}) formalizes this effect by showing that population calibration centers non-member scores around their corresponding neutral baselines.

\paragraph*{Practical Advantages}
Beyond statistical improvements, \plmia provides practical advantages for privacy auditing.
By converting pairwise differences into continuous $p$-values and aggregating them through the Cauchy combination rule, \plmia preserves the strength of individual comparisons.
This design improves sensitivity to informative membership evidence while avoiding the pairwise comparison threshold $\gamma$ required by binary voting.
As shown in \secref{sec:exp_abalation}, PL-MIA achieves the strongest performance without tuning $\gamma$.
Moreover, under the idealized Cauchy null, the Cauchy-combined score $\mathrm{Score_{PL}}(q)$ provides an analytic decision threshold at a target FPR $\alpha$: $\tau_\alpha = \tan(\pi(0.5 - \alpha))$, which eliminates the need for additional held-out data.
As shown in \secref{sec:exp_property}, when the exact-Cauchy null conditions are only approximately satisfied, the analytic threshold provides conservative finite-sample FPR control, with the realized FPR remaining below the nominal level.

\section{Theoretical Results}
\label{sec:theoretical}

In this section, we establish the theoretical advantages of \plmia over \lira and \rmia through a unified analysis of attack power.
We first characterize the GLR and BLR statistics, then examine how population calibration centers non-member scores, and finally compare their attack power.
All results are derived under Assumption~\ref{ass:gaussian_model}, with proofs deferred to Section~S3 of the Supplementary Material.

\subsection{Pointwise Statistic Distributions}
\label{point-sta-dist}

We begin by deriving the GLR and BLR statistics and comparing their statistical efficiency.
We assume a global variance \contract $\rho$, while modeling the normalized mean shift $\eta^{(d)}$ and variation coefficients $\xi^{(d)}$ as random variables with variances $\sigma_\eta^2$ and $\sigma_\xi^2$ to capture query-level heterogeneity.
We quantify statistical efficiency by the \textit{mean separation}, defined as $\Delta\mu^{(d)}=\mathbb{E}[\varphi^{(d)}\mid H_1]-\mathbb{E}[\varphi^{(d)}\mid H_0]$. A larger mean separation indicates greater pointwise distinguishability.

\begin{lemma}[Gaussian LR]
\label{lem:logbf}
Let $W \sim \mathcal{N}(0,1)$. The Gaussian LR statistic takes the form
\begin{align}
    \varphi_{\mathrm{GLR}}^{(d)} \mid H_0
    &= -\log\rho + \frac{W^2}{2} - \frac{(W-\eta^{(d)})^2}{2\rho^2}, \\
    \varphi_{\mathrm{GLR}}^{(d)} \mid H_1
    &= -\log\rho + \frac{(\eta^{(d)}+\rho W)^2}{2} - \frac{W^2}{2}.
\end{align}
Its mean separation is $\Delta\mu_{\mathrm{GLR}}^{(d)} = \frac{(\rho^2-1)^2 + (\eta^{(d)})^2(1+\rho^2)}{2\rho^2}$.
When $\rho < 1$, $\Delta\mu_{\mathrm{GLR}}^{(d)}$ grows as $O(1/\rho^2)$, allowing GLR to capture the signal induced by variance contraction.
\end{lemma}

\begin{lemma}[Bayes LR]
\label{lem:rmia}
Let $W \sim \mathcal{N}(0,1)$. Assuming equal numbers of IN and OUT reference models, a first-order Taylor expansion of the Bayes LR statistic gives
\begin{align}
\varphi_{\mathrm{BLR}}^{(d)} \mid H_1
&\approx \frac{\eta^{(d)}\xi^{(d)}}{2} + \rho\xi^{(d)}W, \\
\varphi_{\mathrm{BLR}}^{(d)} \mid H_0
&\approx -\frac{\eta^{(d)}\xi^{(d)}}{2} + \xi^{(d)}W.
\end{align}
Its first-order mean separation is $\Delta\mu_{\mathrm{BLR}}^{(d)} = \eta^{(d)}\xi^{(d)}$, which is independent of $\rho$.
Thus, BLR does not retain variance contraction as a mean-separation signal.
\end{lemma}

Together, Lemmas~\ref{lem:logbf} and \ref{lem:rmia} identify a key pointwise distinction: GLR retains a variance-contraction signal that is absent from the first-order BLR mean separation. We next study how these statistics behave in pairwise comparisons.

\subsection{Pairwise Difference and Population Calibration}

Population calibration converts pointwise statistics into a membership score by comparing a query point $q$ with population points $z$. To isolate the effect of calibration under query-level heterogeneity ($\sigma_\eta>0$), we set $\rho=1$, thereby removing variance contraction from this comparison.

\begin{lemma}[Pairwise Difference Distribution]
\label{lem:pairwise}
Under $\sigma_\eta>0$ and $\rho=1$, the GLR- and BLR-induced pairwise differences for a member query satisfy
\begin{align}
    & \mathrm{For \ GLR:} \quad \varphi_{\mathrm{GLR}}^{(q)} - \varphi_{\mathrm{GLR}}^{(z)} \sim
    \mathcal{N}\biggl(\frac{(\eta^{(q)})^2 + (\eta^{(z)})^2}{2}, (\eta^{(q)})^2 + (\eta^{(z)})^2\biggr). \\
    & \mathrm{For \ BLR:} \quad \varphi_{\mathrm{BLR}}^{(q)} - \varphi_{\mathrm{BLR}}^{(z)} \sim
    \mathcal{N}\biggl(\frac{\eta^{(q)}\xi^{(q)} + \eta^{(z)}\xi^{(z)}}{2}, (\xi^{(q)})^2 + (\xi^{(z)})^2\biggr).
\end{align}
The corresponding non-member distributions are presented in Section~S3.3 of the Supplementary Material.
\end{lemma}

Population calibration aggregates these pairwise differences into a final membership score. We consider the binary-voting score in \eqref{eq:binary_calib} and the Cauchy-combined \plmia score in \eqref{eq:cauchy_cali}.
For binary voting, we consider the neutral threshold $\gamma=0$ that asks whether the query point wins against a population point.

\begin{lemma}[Membership Scores under Population Calibration]
\label{lem:score}
Let $Q$ be a random non-member query point. The two population-calibrated membership scores have the following neutral baselines:
\begin{align}
    & \mathrm{Binary\ voting:} \quad
    \mathbb{E}\left[\mathrm{Score}_{\mathrm{Binary}}(Q)\mid H_0\right]
    =0.5, \label{eq:binary_null_centering} \\
    & \mathrm{Cauchy\ combination:} \quad
    \mathrm{Median}\left(\mathrm{Score}_{\mathrm{PL}}(Q)\mid H_0\right)
    \approx 0. \label{eq:cauchy_null_centering}
\end{align}
Under binary voting, exchangeability gives a non-member query an average win probability of $0.5$.
Under Cauchy combination, null pairwise $p$-values are mapped to median-zero Cauchy evidence, giving the aggregated score median zero.
We use the median because Cauchy evidence has no finite mean.
\end{lemma}

\subsection{Attack-Power Expressions and Advantage Analysis}

We use attack power~\citep{zhu2025impact} to compare attacks, defined as the TPR at a fixed FPR.
To separate the contributions of pointwise statistics and evidence aggregation, let $\mathrm{PL\text{-}MIA}^{+}$ denote the intermediate attack that uses the GLR pointwise statistic and binary-voting population calibration.
\plmia is obtained from $\mathrm{PL\text{-}MIA}^{+}$ by replacing binary voting with the Cauchy combination.

For \lira, \rmia, and \plmiap, their Gaussian score distributions yield a common CDF-based expression for attack power.
For \plmia, the heavy-tailed Cauchy-combined score instead yields a low-FPR characterization in terms of exceptionally small pairwise $p$-values.

\begin{theorem}[Attack-Power Expressions for the Compared Attacks]
\label{thm:unified_tpr}
Under $\rho=1$, at a fixed FPR $\alpha$, the attack power for \lira, \rmia, and $\mathrm{PL\text{-}MIA}^{+}$ has the unified form
\begin{align}
\mathrm{TPR}_{A}(\alpha)
&=\Phi\left(\kappa_A-\Phi^{-1}(1-\alpha)\right), \quad A\in\{\mathrm{LiRA},\mathrm{RMIA},\mathrm{PL\text{-}MIA}^{+}\}. \label{eq:tpr_gaussian_scores}
\end{align}
Here $\Phi(\cdot)$ and $\Phi^{-1}(\cdot)$ denote the standard Gaussian CDF and its inverse. The parameter $\kappa_A$ measures the standardized separation between the member and non-member score distributions for attack $A$, so a larger $\kappa_A$ yields higher attack power.
For \plmia, define
\begin{align}
u_{\alpha,|Z|}:= \frac{1}{2} -\frac{1}{\pi}\arctan\left(|Z|\tau_\alpha\right) \approx \frac{\alpha}{|Z|}, \quad
r_\alpha := \Pr\left[p(Q,Z)\leq u_{\alpha,|Z|}\mid H_1\right],
\label{eq:small_p_probability}
\end{align}
where $\tau_\alpha = \tan\left(\pi(0.5-\alpha)\right)$ is the analytic decision threshold.
Here, $u_{\alpha,|Z|}$ is the pairwise $p$-value level corresponding to the analytic threshold, and $r_\alpha$ is the probability that a member query reaches this level in one comparison. Aggregating over the $|Z|$ comparisons gives
\begin{align}
\mathrm{TPR}_{\mathrm{PL\text{-}MIA}}(\alpha)
&\approx
1-(1-r_\alpha)^{|Z|}.
\label{eq:tpr_cauchy_plmia}
\end{align}
This expression captures the probability that, among $|Z|$ pairwise comparisons, at least one comparison produces strong membership evidence.
Lemma~\ref{lem:pairwise} explains why this event is more likely for member queries: larger pairwise differences translate into smaller $p$-values.
\end{theorem}


\begin{theorem}[Advantage Analysis]
\label{thm:main}
Under Assumption~\ref{ass:gaussian_model} and at a fixed $\alpha$, the three components of PL-MIA yield the following attack-power comparisons:
\begin{itemize}
    \item[(I)] If $\sigma_\eta^2>0$ and $\rho = 1$, then $\mathrm{TPR}_{\mathrm{PL\text{-}MIA}^{+}}(\alpha) = \mathrm{TPR}_{\mathrm{LiRA}}(\alpha)$.
    Population calibration additionally places the GLR scores on a common null-relative scale.
    \item[(II)] If $\rho<1$, or if $\rho = 1$ and $\sigma_\xi^2>0$, then $\mathrm{TPR}_{\mathrm{PL\text{-}MIA}^{+}}(\alpha)>\mathrm{TPR}_{\mathrm{RMIA}}(\alpha)$.
    \item[(III)] If $|Z|r_\alpha>-\log\left(1-\mathrm{TPR}_{\mathrm{PL\text{-}MIA}^{+}}(\alpha)\right)$, then $\mathrm{TPR}_{\mathrm{PL\text{-}MIA}}(\alpha)>\mathrm{TPR}_{\mathrm{PL\text{-}MIA}^{+}}(\alpha)$.
\end{itemize}
Part (I) characterizes the role of population calibration: it preserves LiRA's attack power while placing heterogeneous GLR scores on a common relative scale, thereby reducing query-level heterogeneity.
Part (II) quantifies the gain from GLR: unlike BLR, it retains the variance-contraction signal when $\rho<1$ and avoids the query-dependent multiplicative variation of the first-order BLR score when $\rho=1$ and $\sigma_\xi^2>0$.
Part (III) identifies when Cauchy combination outperforms binary voting by retaining the evidence strength and amplifying exceptionally small $p$-values.

Together, these results decompose the advantage of \plmia into pointwise signal extraction, population calibration, and evidence-sensitive aggregation.
\end{theorem}

\section{Experiments}
\label{sec:exp}

\subsection{Experimental Setup}

\paragraph*{Datasets and Implementation Details}
Our main experiments evaluate \plmia on three image datasets: \cifar{}, \cifarb~\citep{cifar}, and \cinic{}~\citep{cinic}.
For a fair comparison, we use \wrn~\citep{wideresnet} as the backbone across all datasets and follow the model-training protocol of \citet{rmia}.
Additional details on the model-training protocol and hyperparameters are provided in Section~S1 of the Supplementary Material.
Following the standard evaluation protocol~\citep{lira,wen2023canary}, we construct balanced member and non-member evaluation sets by randomly partitioning each dataset into two disjoint halves of equal size.
One half is used to train the target model and constitutes the member set, while the other half is held out as the non-member set.
We use all points in both sets as audit queries.
We train multiple reference models~\citep{emia,watson2021importance} using the same \wrn{} backbone and training protocol. 
Each reference model is trained on an independently sampled $50\%$ subset of the full dataset, so that each query point is included in approximately half of the reference-model training sets.
For each query point, reference models trained on that point are referred to as IN reference models, while those not trained on it are referred to as OUT reference models.

\paragraph*{Attack Mode (Online vs. Offline) and Baselines}
We consider both online and offline attack modes.
In the online setting, the adversary uses both IN and OUT reference models, whereas the offline setting permits access only to OUT reference models.
Under both attack modes, we compare \plmia with representative baselines, including Attack-R, Attack-P~\citep{emia}, \lira~\citep{lira}, and \rmia~\citep{rmia}.
For a fair comparison, we apply the same data-augmentation protocol of \rmia~\citep{rmia} to all compatible attacks.
In our experiments, \lira and \rmia are evaluated in both online and offline modes, whereas Attack-P and Attack-R are restricted to the offline mode, as they rely on global statistics (e.g., loss) that do not condition on the inclusion of the query point.

\paragraph*{Evaluation Metrics}
We evaluate each attack using two standard metrics: the area under the receiver operating characteristic curve (AUC) and the true positive rate (TPR) at extremely low false positive rates (FPRs). Specifically, we report TPR at FPRs of 0.01\% and 0.0\%.
AUC measures the attack's overall discriminative ability by summarizing the trade-off between TPR and FPR across all possible decision thresholds.
In contrast, TPR at low FPRs measures attack effectiveness under stringent false-positive constraints, which is particularly relevant to privacy audits that require a near-zero false-accusation rate.

\begin{table*}[tbp]
\centering
\newcommand{\msd}[2]{\ensuremath{#1_{\scriptscriptstyle #2}}}
\newcommand{\bmsd}[2]{\ensuremath{\mathbf{#1}_{\scriptscriptstyle #2}}}
\newcommand{\umsd}[2]{\ensuremath{\underline{#1}_{\scriptscriptstyle #2}}}
\caption{Performance of different attacks using 254 reference models, averaged over 10 random target models.
For each scenario, the best-performing method is highlighted in \textbf{bold}, while the second-best result is \underline{underlined}.
Standard deviations over target models are shown as subscripts.
The top part corresponds to offline attacks using only 127 OUT models, whereas the bottom part reports online attacks using 127 OUT and 127 IN models.
\textbf{IoS} (Improvement over SOTA) reports the improvement of \plmia over the previously strongest baseline under the same setting, with absolute and relative gains shown in IoS (Abs.) and IoS (Rel.), respectively.}
\label{tab:main_results}
\resizebox{\textwidth}{!}{
\begin{tabular}{lccccccccc}
\toprule
\multirow{2}{*}{Attack} &
\multicolumn{3}{c}{CIFAR-10} & \multicolumn{3}{c}{CIFAR-100} &
\multicolumn{3}{c}{CINIC-10} \\
\cmidrule(lr){2-4} \cmidrule(lr){5-7} \cmidrule(lr){8-10}
& AUC & \multicolumn{2}{c}{TPR@FPR}
& AUC & \multicolumn{2}{c}{TPR@FPR}
& AUC & \multicolumn{2}{c}{TPR@FPR} \\
\cmidrule(lr){3-4} \cmidrule(lr){6-7} \cmidrule(lr){9-10}
& & 0.01\% & 0.0\% & & 0.01\% & 0.0\% & & 0.01\% & 0.0\% \\
\midrule
\rowcolor{gray!25}
\multicolumn{10}{c}{Offline} \\
Attack-P & \msd{58.15}{0.45} & \msd{0.00}{0.00} & \msd{0.00}{0.00} & \msd{77.12}{0.19} & \msd{0.01}{0.01} & \msd{0.00}{0.00} & \msd{66.30}{0.23} & \msd{0.00}{0.01} & \msd{0.00}{0.00} \\
Attack-R & \msd{64.24}{0.39} & \msd{1.33}{0.24} & \msd{0.81}{0.31} & \msd{83.30}{0.22} & \msd{5.92}{1.70} & \msd{3.12}{1.02} & \msd{72.86}{0.18} & \msd{1.70}{0.58} & \msd{1.03}{0.68} \\
\lira & \msd{55.33}{0.46} & \msd{1.33}{0.34} & \msd{0.71}{0.35} & \msd{75.87}{0.32} & \msd{2.73}{1.28} & \msd{1.52}{1.26} & \msd{64.04}{0.59} & \msd{1.31}{0.26} & \msd{0.60}{0.39} \\
\rmia & \bmsd{71.64}{0.34} & \umsd{4.03}{0.82} & \umsd{3.08}{0.62} & \umsd{90.60}{0.13} & \umsd{10.79}{2.77} & \umsd{7.73}{2.47} & \bmsd{82.11}{0.30} & \umsd{5.46}{0.95} & \umsd{3.61}{1.18} \\
\plmia & \umsd{71.18}{0.37} & \bmsd{4.56}{0.79} & \bmsd{3.34}{0.77} & \bmsd{91.22}{0.13} & \bmsd{15.82}{2.23} & \bmsd{12.58}{2.89} & \umsd{81.45}{0.28} & \bmsd{7.04}{0.92} & \bmsd{5.44}{1.41} \\
\rowcolor{blue!15}
IoS (Abs.) & $-0.46$ & $0.53$ & $0.26$ & $0.62$ & $5.03$ & $4.85$ & $-0.66$ & $1.58$ & $1.83$ \\
\rowcolor{blue!15}
IoS (Rel.) & $\downarrow$0.64\% & $\uparrow$13.15\% & $\uparrow$8.44\% & $\uparrow$0.68\% & $\uparrow$46.62\% & $\uparrow$62.74\% & $\downarrow$0.80\% & $\uparrow$28.94\% & $\uparrow$50.69\% \\
\midrule
\rowcolor{gray!25}
\multicolumn{10}{c}{Online} \\
\lira & \msd{72.10}{0.41} & \msd{3.27}{0.42} & \msd{1.96}{0.91} & \umsd{91.48}{0.13} & \msd{11.60}{1.84} & \msd{7.61}{2.94} & \msd{82.23}{0.28} & \msd{4.78}{1.34} & \msd{3.01}{1.36} \\
\rmia & \umsd{72.19}{0.35} & \umsd{4.35}{0.68} & \umsd{2.79}{0.66} & \msd{91.00}{0.14} & \umsd{12.66}{2.29} & \umsd{10.37}{2.67} & \umsd{82.48}{0.32} & \umsd{6.83}{1.42} & \umsd{4.34}{1.68} \\
\plmia & \bmsd{72.49}{0.34} & \bmsd{4.45}{0.54} & \bmsd{3.43}{0.80} & \bmsd{92.10}{0.14} & \bmsd{16.81}{2.41} & \bmsd{13.20}{2.89} & \bmsd{82.71}{0.27} & \bmsd{7.12}{1.01} & \bmsd{5.54}{1.22} \\
\rowcolor{blue!15}
IoS (Abs.) & $0.30$ & $0.10$ & $0.64$ & $0.62$ & $4.15$ & $2.83$ & $0.23$ & $0.29$ & $1.20$ \\
\rowcolor{blue!15}
IoS (Rel.) & $\uparrow$0.42\% & $\uparrow$2.30\% & $\uparrow$22.94\% & $\uparrow$0.68\% & $\uparrow$32.80\% & $\uparrow$27.29\% & $\uparrow$0.28\% & $\uparrow$4.25\% & $\uparrow$27.65\% \\
\bottomrule
\end{tabular}}
\end{table*}

\subsection{Main Results}

We compare \plmia against baseline attacks across datasets and attack modes, with the results reported in \tabref{tab:main_results}.
In the online setting, \plmia consistently achieves the strongest performance for all datasets and metrics.
The gains are most pronounced in the low-FPR regime: at TPR@0.0\%FPR, \plmia improves over the strongest baseline by 22.94\%, 27.29\%, and 27.65\% on \cifar{}, \cifarb, and \cinic, respectively.
In the offline setting, \plmia remains competitive in AUC and achieves the strongest low-FPR performance across all datasets. In particular, on \cifarb and \cinic, \plmia improves TPR@0.0\%FPR over the strongest baseline by \textbf{62.74\%} and \textbf{50.69\%}, respectively.
These results support our theoretical findings in \secref{sec:theoretical} that \plmia yields a more powerful test in the practically critical low-FPR regime.

\subsection{Analysis}
We further analyze the component contributions, sensitivity, computational cost, and statistical properties of \plmia.
Unless otherwise specified, we report the results on the image datasets.

\subsubsection{Ablation and Sensitivity Analyses}\label{sec:exp_abalation}

\paragraph*{Component-wise Ablation}
We ablate each component of \plmia across datasets and attack modes, and report the results in \tabref{tab:component_analysis}.
Row 1 uses only the Gaussian likelihood ratio (GLR) as the pointwise statistic, corresponding to \lira.
Row 2 adds binary-voting population calibration to the GLR, yielding the intermediate attack \plmiap.
This addition substantially improves TPR@0.0\%FPR in all settings, consistent with our analysis that population calibration neutralizes query-level heterogeneity (Lemma~\ref{lem:score}).
Finally, Row 3 uses the continuous pairwise $p$-values and aggregates them using the Cauchy combination (CC), yielding the complete \plmia.
This final step provides a further improvement and achieves the strongest performance under all reported settings.

\begin{table}[!htbp]
\centering
\small
\renewcommand{\tabcolsep}{0.3pc}
\renewcommand{\arraystretch}{0.5}
\caption{Component-wise ablation of \plmia, evaluated using TPR@0\%FPR across datasets and attack modes. GLR, PC, and CC denote Gaussian Likelihood Ratio, Population Calibration, and Cauchy Combination, respectively.}
\label{tab:component_analysis}
\begin{tabular}{ccc|cc|cc|cc}
\toprule
\multicolumn{3}{c|}{Component} &
\multicolumn{2}{c|}{\cifar} &
\multicolumn{2}{c|}{\cifarb} &
\multicolumn{2}{c}{\cinic} \\
\cmidrule(lr){4-5} \cmidrule(lr){6-7} \cmidrule(lr){8-9}
GLR & PC & CC & Online & Offline & Online & Offline & Online & Offline \\
\midrule
\ding{52} &         &         & 1.96 & 0.71 & 7.61 & 1.52 & 3.01 & 0.60 \\
\ding{52} & \ding{52} &       & 3.32 & 3.29 & 12.82 & 12.32 & 5.40 & 5.16 \\
\ding{52} & \ding{52} & \ding{52} & \textbf{3.43} & \textbf{3.34} & \textbf{13.20} & \textbf{12.58}  & \textbf{5.54} & \textbf{5.44 }\\
\bottomrule
\end{tabular}
\end{table}

\paragraph*{Sensitivity to Pairwise Comparison Threshold $\gamma$}
We further study the sensitivity to the pairwise comparison threshold $\gamma$ in \eqref{eq:binary_calib}, which determines whether each pairwise comparison contributes positive binary evidence.
As shown in \figref{fig:abla_gamma}, \rmia is highly sensitive to $\gamma$ and underperforms \lira when $\gamma = 1$.
For comparison, we also report the performance of \plmiap, which uses the GLR pointwise statistic with the same binary-voting population calibration.
It consistently outperforms \rmia for all considered $\gamma$ and is substantially less sensitive to the threshold choice.
The complete \plmia does not require $\gamma$ because it converts pairwise differences into continuous $p$-values and aggregates them using the Cauchy combination.
It therefore eliminates the need to tune the hyperparameter $\gamma$ while achieving the strongest performance.

\begin{figure}[!ht]
    \centering
    \begin{overpic}[width=0.5\linewidth]{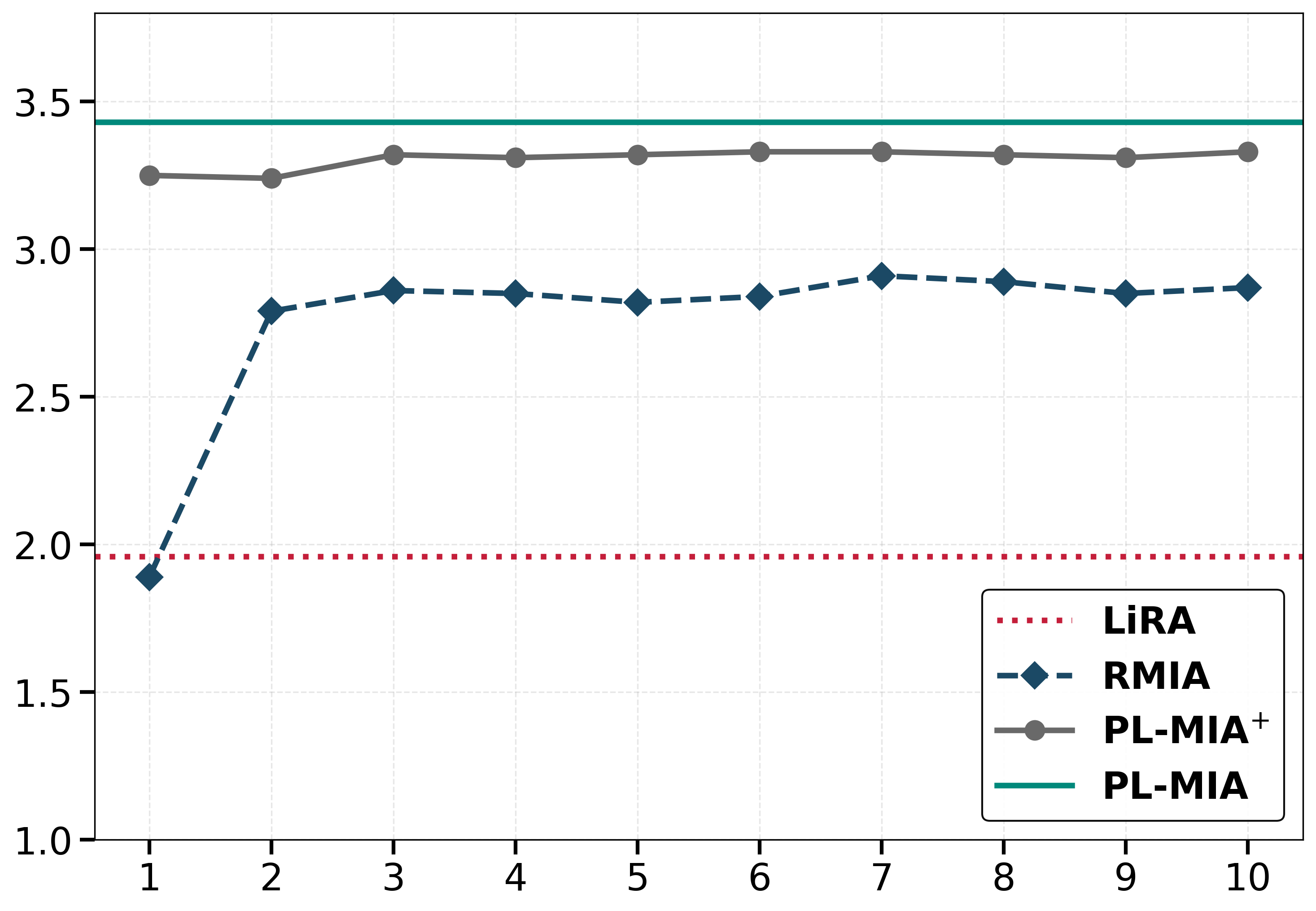}
        \put(45,-2.5){Threshold $\gamma$}      
        \put(-6,30){\rotatebox{90}{TPR\@0.0\%FPR}} 
    \end{overpic}
    \vspace{1em}
    \caption{Sensitivity to the pairwise comparison threshold $\gamma$ on \cifar, evaluated using TPR@0.0\% FPR.}
    \label{fig:abla_gamma}
\end{figure}

\paragraph*{Offline \plmia}\label{sec:exp_offline}

In the offline setting, the adversary has access only to OUT reference models and therefore cannot directly estimate the IN distribution.
To enable \plmia in this setting, we estimate the IN-distribution parameters from the OUT-distribution statistics, based on the empirical observation that members typically exhibit higher confidence (positive mean shift) and lower uncertainty (variance contraction).
Specifically, we approximate the IN-distribution parameters as
\begin{equation}
\begin{aligned}
    \sigma_{\mathrm{in}} \approx \rho \cdot \sigma_{\mathrm{out}}, \quad
    \mu_{\mathrm{in}} \approx \mu_{\mathrm{out}} + \eta \cdot \sigma_{\mathrm{out}},
\end{aligned}
\end{equation}
where $\rho \in (0, 1]$ is the variance-contraction ratio and $\eta > 0$ is the normalized mean shift.
\tabref{tab:offline_all} shows that \plmia remains stable across the considered values of $\eta$ and $\rho$.
For example, on CIFAR-100, AUC ranges only from $90.12$ to $91.22$, while TPR@0.0\%FPR ranges from $12.14$ to $12.78$, with $\eta=0.3$ and $\rho=0.9$ achieving the best trade-off.
These results indicate that even without explicit IN reference models, approximating the normalized mean shift and variance contraction allows \plmia to maintain strong attack performance.

\begin{table*}[!htbp]
\centering
\renewcommand{\tabcolsep}{0.3pc}
\renewcommand{\arraystretch}{0.5}
\caption{Offline \plmia performance under different normalized mean shifts $\eta$ and variance-contraction ratios $\rho$. For each dataset, the best result for each metric is highlighted in \textbf{bold}.}
\label{tab:offline_all}
\begin{tabular}{cc|ccc|ccc|ccc}
\toprule
$\eta$ & $\rho$
& \multicolumn{3}{c}{CIFAR-10}
& \multicolumn{3}{c}{CIFAR-100}
& \multicolumn{3}{c}{CINIC-10} \\
\cmidrule(lr){3-5} \cmidrule(lr){6-8} \cmidrule(lr){9-11}
 &  & AUC & \multicolumn{2}{c}{TPR@FPR}
      & AUC & \multicolumn{2}{c}{TPR@FPR}
      & AUC & \multicolumn{2}{c}{TPR@FPR} \\
\cmidrule(lr){4-5} \cmidrule(lr){7-8} \cmidrule(lr){10-11}
 &  &  & 0.01\% & 0.0\%
      & & 0.01\% & 0.0\%
      & & 0.01\% & 0.0\% \\
\midrule
\multirow{2}{*}{0.3} & 0.8
& \textbf{71.18} & \textbf{4.56} & \textbf{3.34}
& 91.08 & 16.03 & 12.78
& 81.17 & 7.01 & 5.34 \\
 & 0.9
& 71.02 & 4.52 & 3.33
& \textbf{91.22} & \textbf{15.82} & \textbf{12.58}
& \textbf{81.45} & \textbf{7.04} & \textbf{5.44} \\
\midrule
\multirow{2}{*}{0.5} & 0.8
& 70.19 & 4.52 & 3.34
& 90.35  & 15.96 & 12.14
& 81.01 & 7.04 & 5.44 \\
 & 0.9
& 70.12 & 4.54 & 3.33
& 90.37  & 15.92 & 12.32
& 80.57  & 7.16 & 5.24 \\
\midrule
\multirow{2}{*}{1.0} & 0.8
& 69.12 & 4.52 & 3.34
& 90.22 & 16.03 & 12.78
& 79.90 & 7.01 & 5.34 \\
 & 0.9
& 69.06 & 4.55 & 3.33
& 90.12 & 15.82 & 12.58
& 79.87 & 7.04 & 5.44 \\
\bottomrule
\end{tabular}
\end{table*}

\subsubsection{Cost Analysis}\label{sec:exp_cost}
We analyze the cost-performance trade-off by varying three computational factors: the number of reference models, the number of data augmentations, and the number of population points.

\begin{figure*}[!htbp]
\centering
\captionsetup[subfigure]{labelformat=parens}

\newcommand{\rowH}{3.4cm}          
\setlength{\tabcolsep}{1pt}        
\renewcommand{\arraystretch}{1.0}  
\newcommand{\labW}{0.9em}          
\newcommand{\imgW}{0.31\linewidth} 

\begin{tabular}{@{}c@{\hspace{2pt}}ccc@{}}

\parbox[c][\rowH][c]{\labW}{\centering\rotatebox{90}{AUC}} &
\begin{subfigure}[c]{\imgW}\centering
  \includegraphics[width=\linewidth,height=\rowH,keepaspectratio]{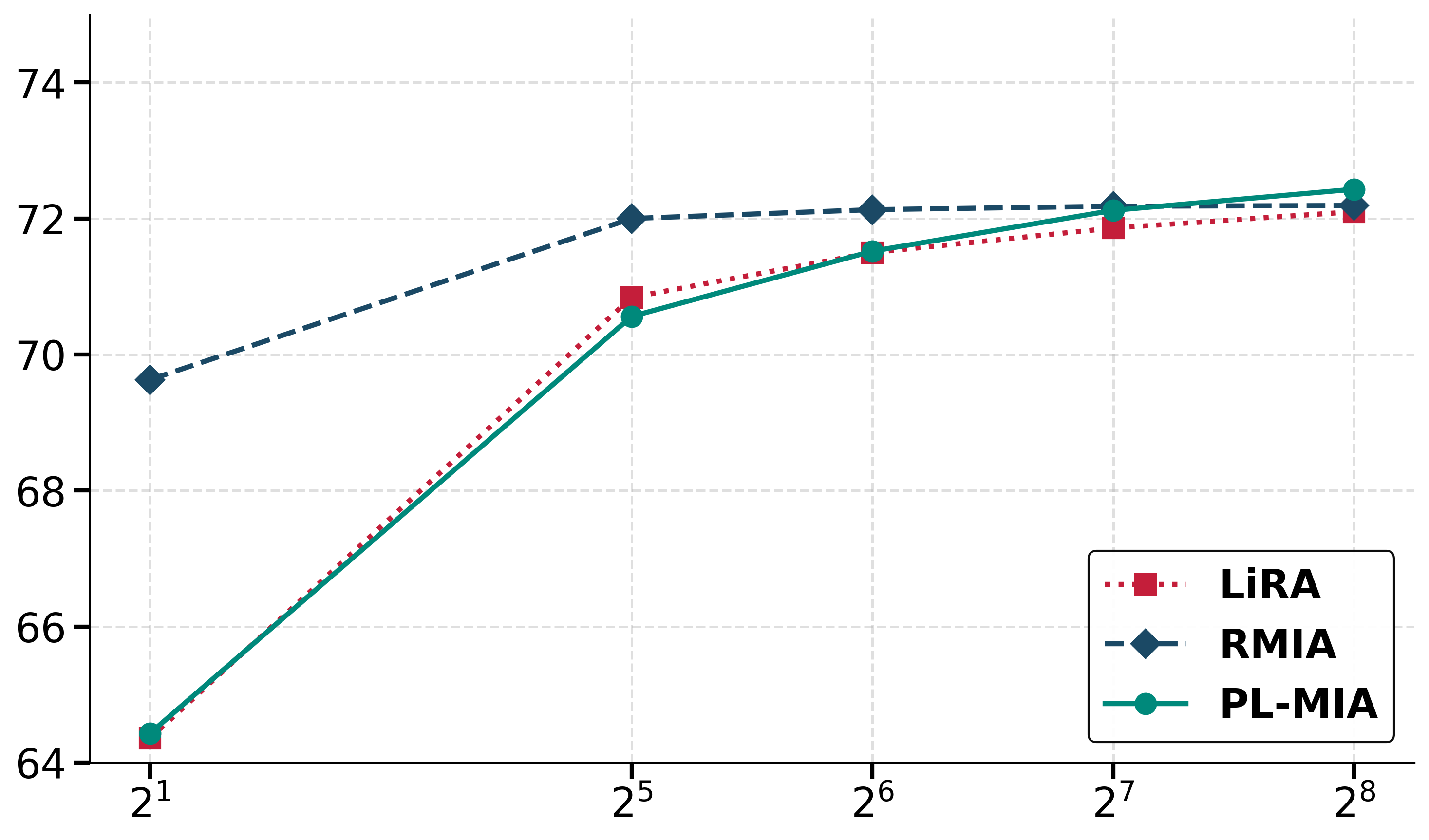}
\end{subfigure} &
\begin{subfigure}[c]{\imgW}\centering
  \includegraphics[width=\linewidth,height=\rowH,keepaspectratio]{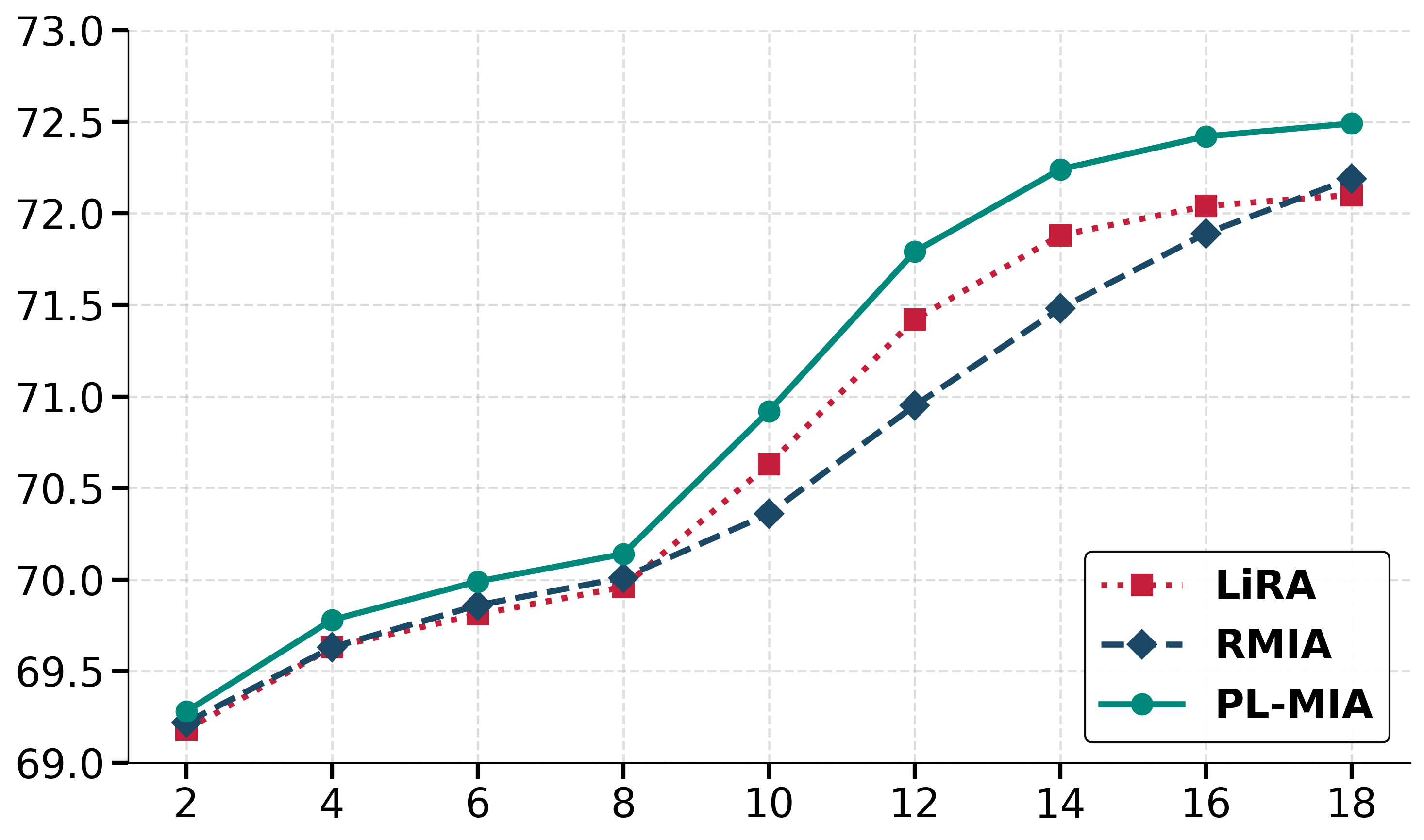}
\end{subfigure} &
\begin{subfigure}[c]{\imgW}\centering
  \includegraphics[width=\linewidth,height=\rowH,keepaspectratio]{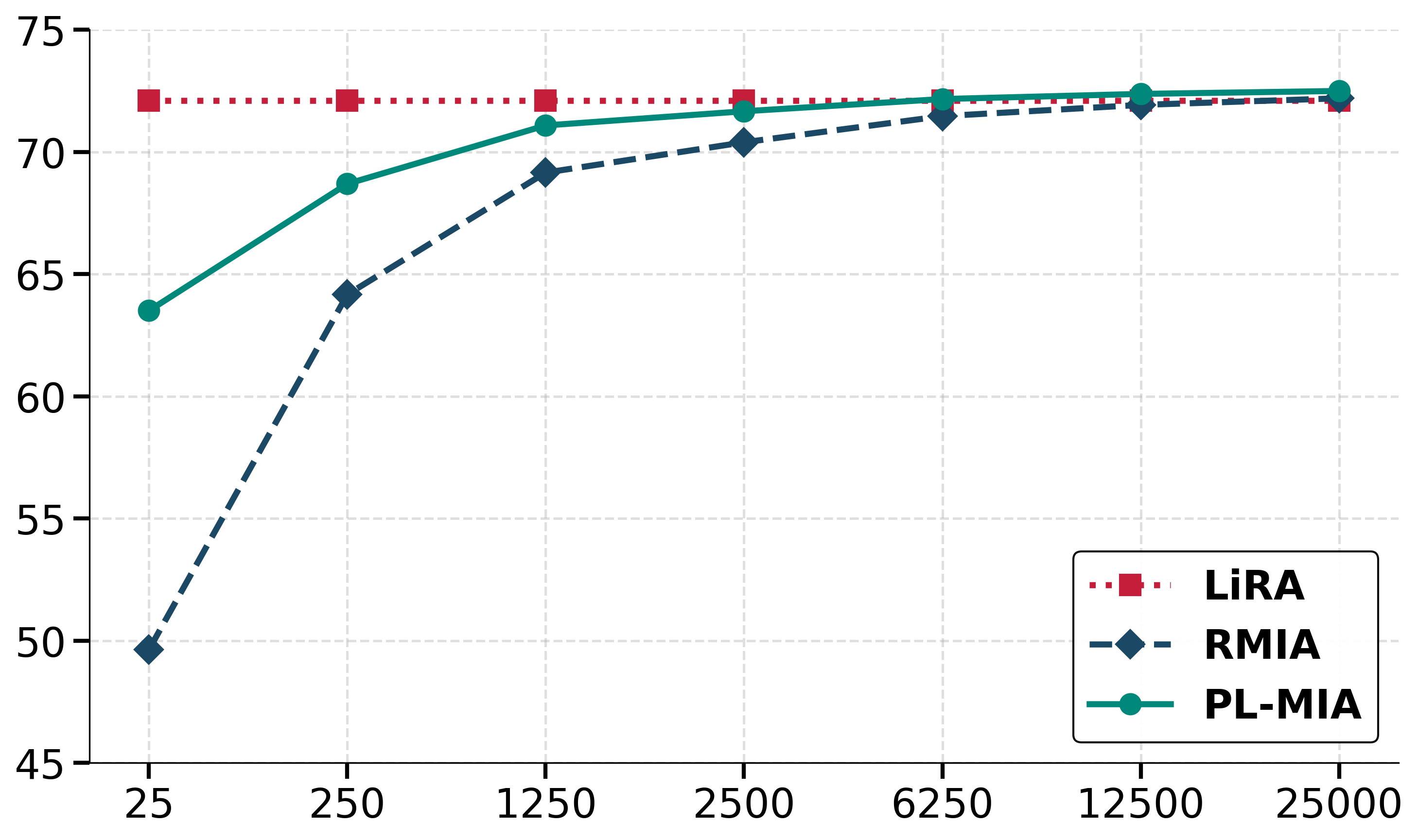}
\end{subfigure}
\\[0.6em]

\parbox[c][\rowH][c]{\labW}{\centering\rotatebox{90}{TPR@0.01\%FPR}} &
\begin{subfigure}[c]{\imgW}\centering
  \includegraphics[width=\linewidth,height=\rowH,keepaspectratio]{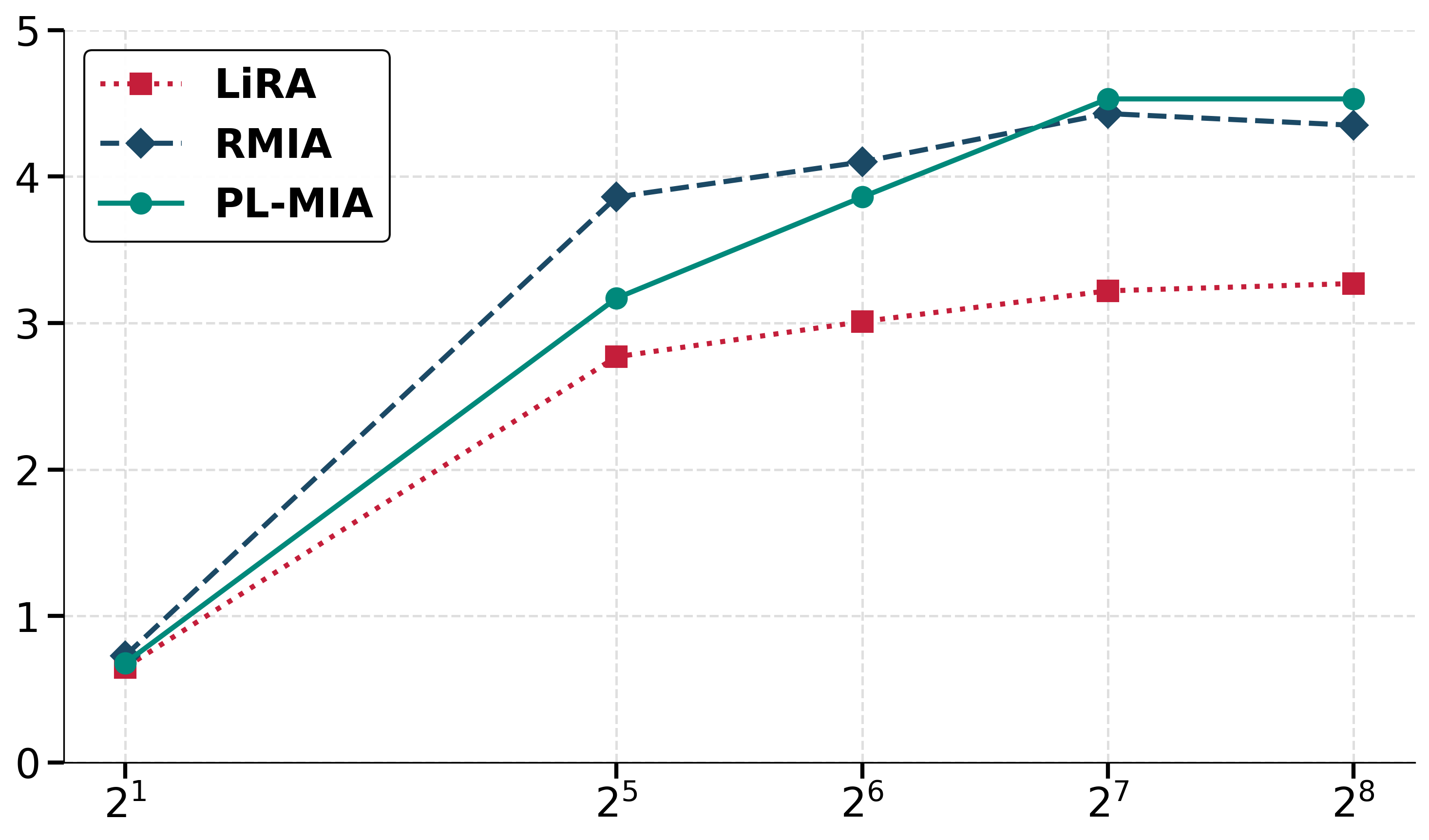}
\end{subfigure} &
\begin{subfigure}[c]{\imgW}\centering
  \includegraphics[width=\linewidth,height=\rowH,keepaspectratio]{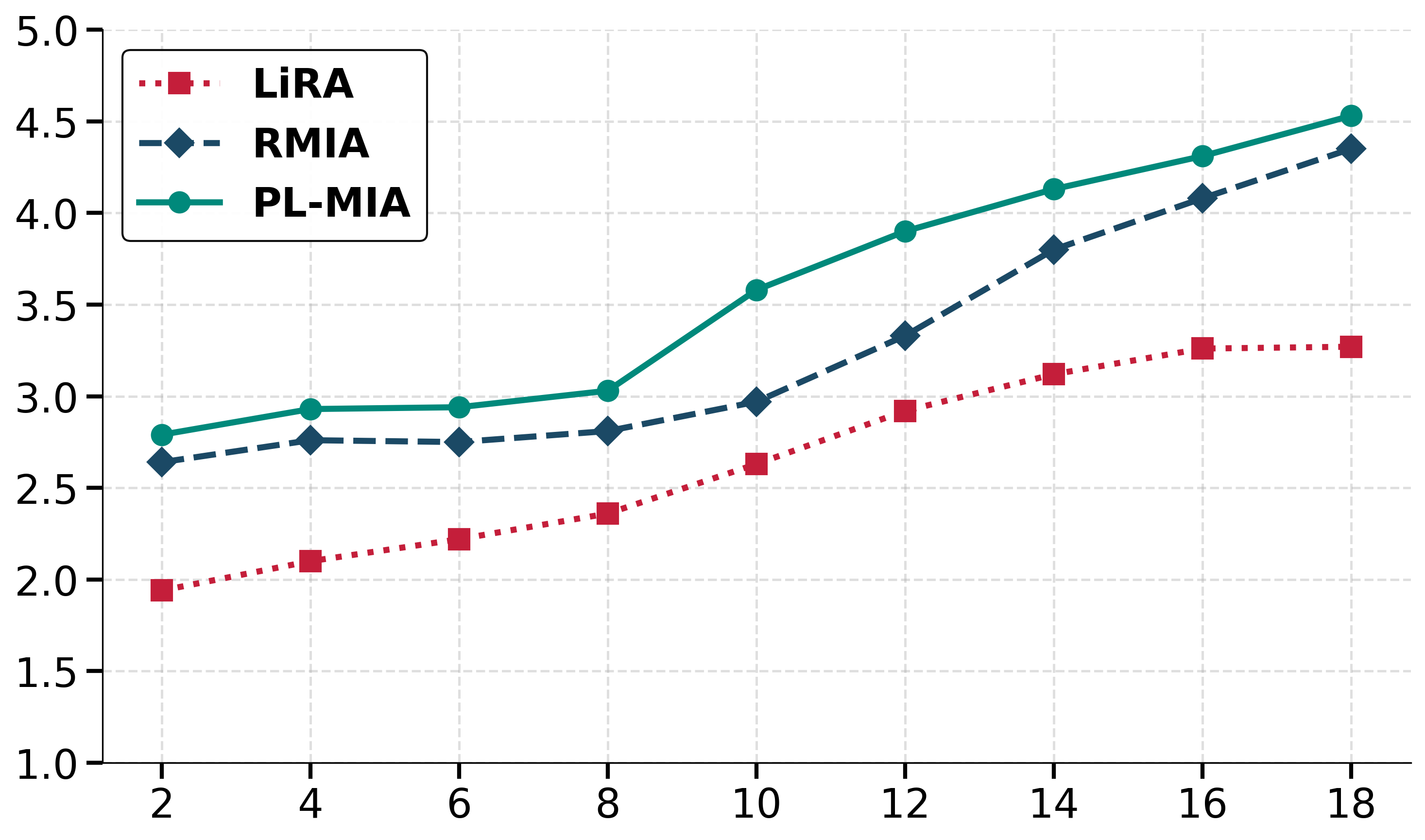}
\end{subfigure} &
\begin{subfigure}[c]{\imgW}\centering
  \includegraphics[width=\linewidth,height=\rowH,keepaspectratio]{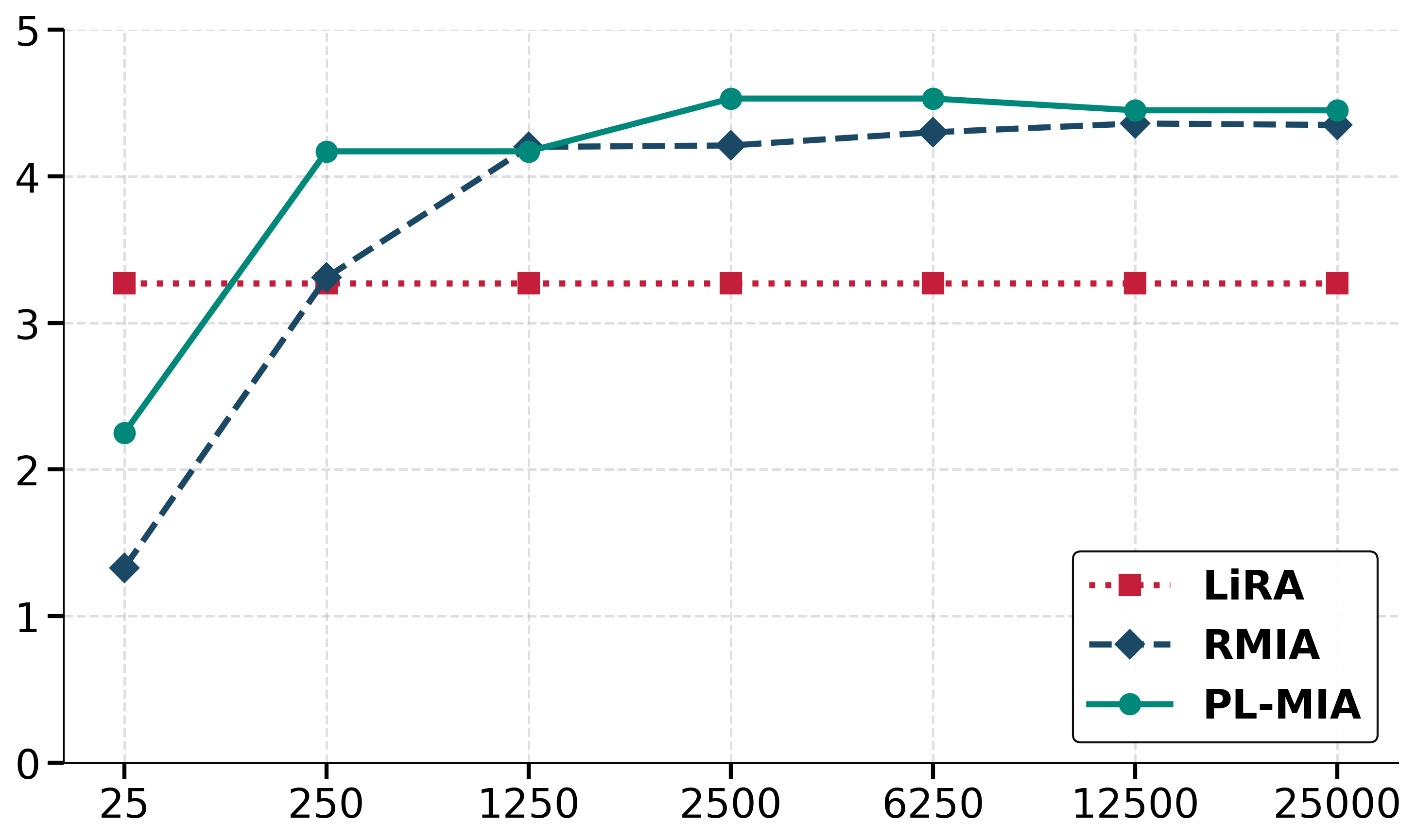}
\end{subfigure}
\\[0.6em]

\parbox[c][\rowH][c]{\labW}{\centering\rotatebox{90}{TPR@0.0\%FPR}} &
\begin{subfigure}[c]{\imgW}\centering
  \includegraphics[width=\linewidth,height=\rowH,keepaspectratio]{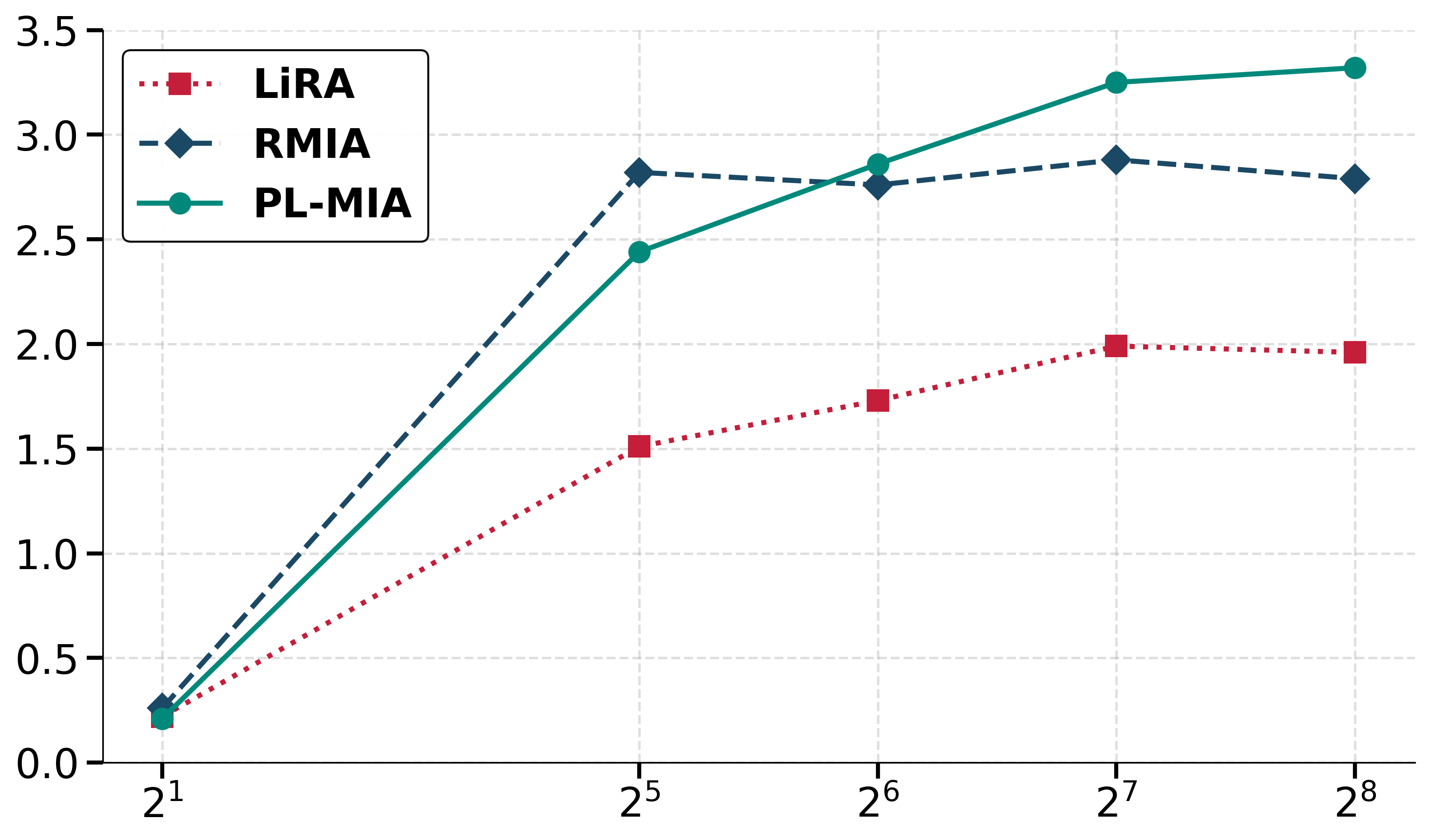}
  \caption{\# reference models}
  \label{fig:ablation-ref}
\end{subfigure} &
\begin{subfigure}[c]{\imgW}\centering
  \includegraphics[width=\linewidth,height=\rowH,keepaspectratio]{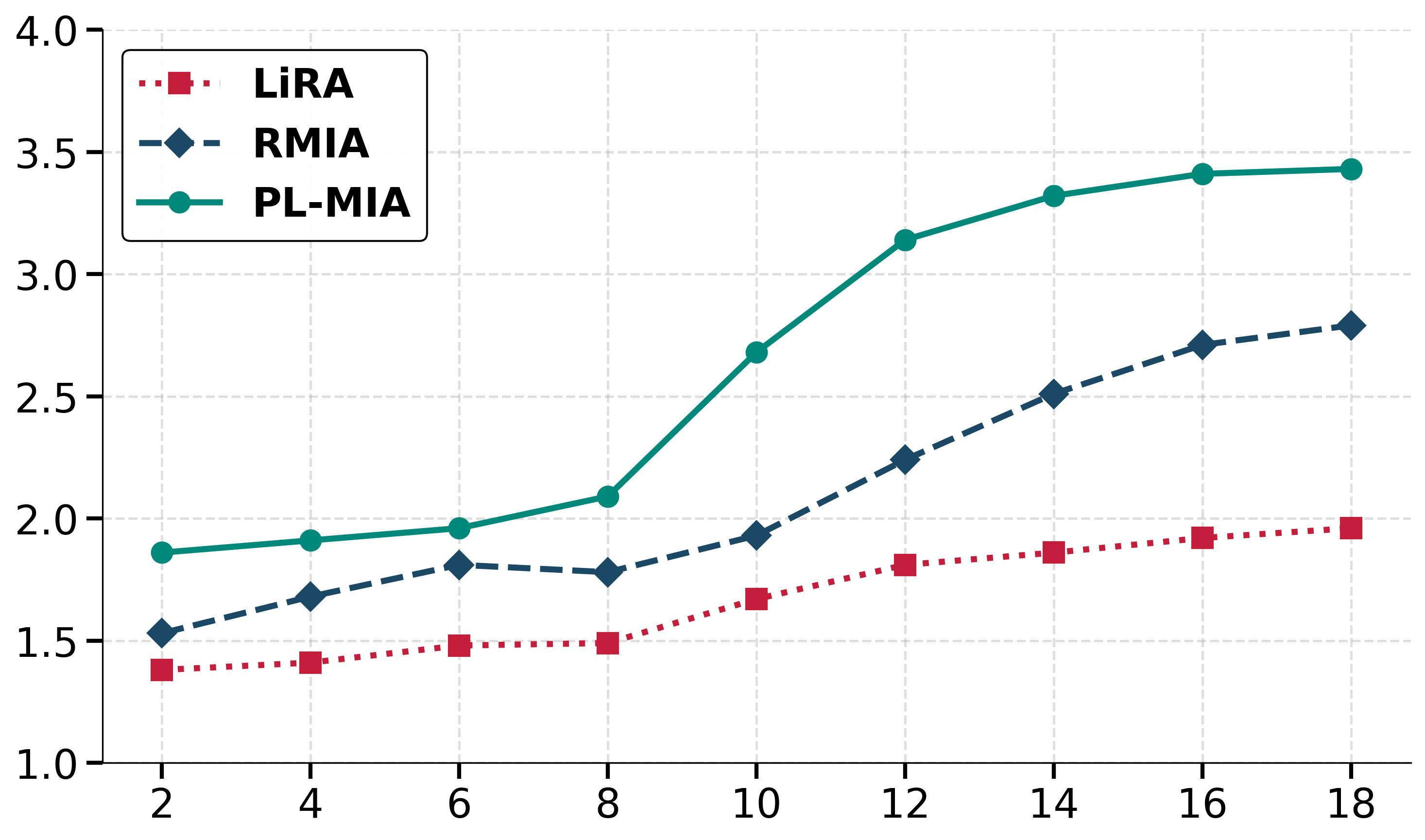}
  \caption{\# data augmentations}
  \label{fig:ablation-aug}
\end{subfigure} &
\begin{subfigure}[c]{\imgW}\centering
  \includegraphics[width=\linewidth,height=\rowH,keepaspectratio]{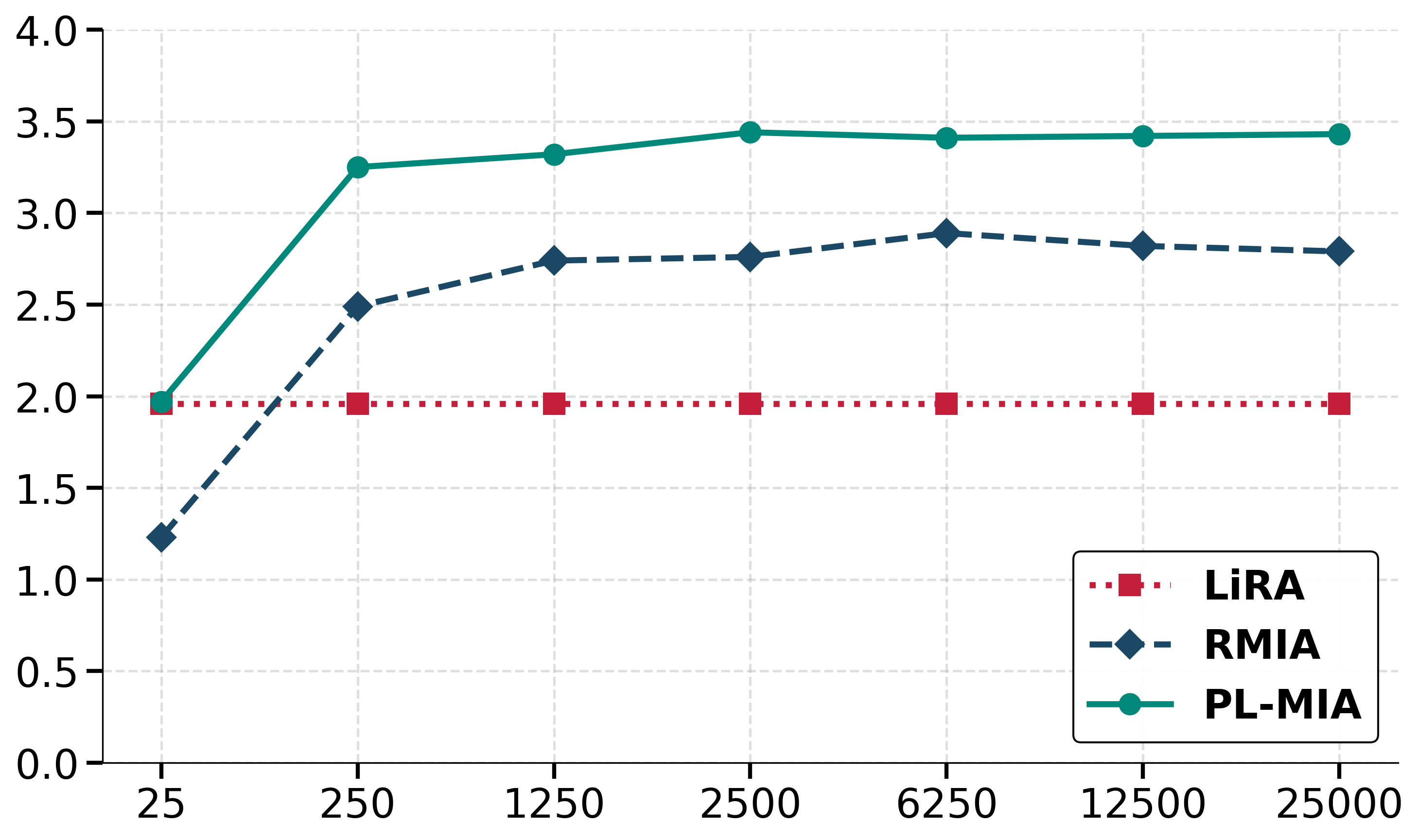}
  \caption{\# population samples}
  \label{fig:ablation-pop}
\end{subfigure}
\\
\end{tabular}
\caption{Cost-performance analysis on \cifar{}. Rows report AUC, TPR@0.01\%FPR, and TPR@0.0\%FPR, while columns vary the number of reference models, data augmentations, and population samples, respectively.}
\label{fig:ablation-3x3}
\end{figure*}

\paragraph*{Number of Reference Models}
We first vary the number of reference models to study their effect on the performance of different attacks.
As shown in \figref{fig:ablation-ref}, all attacks benefit from an increasing number of reference models.
However, \plmia is comparatively sensitive in the smallest-budget regime because both the IN- and OUT-distribution Gaussian parameters must be estimated from only a few reference-model observations. In particular, reliable estimation of the variances requires a sufficiently large reference-model set.
As the reference-model budget increases, these estimates become more stable, and \plmia rapidly closes the gap with \rmia before eventually outperforming the baselines.
An additional comparison under limited reference-model budgets is reported in Section~S2.1 of the Supplementary Material.
Overall, these results indicate that \plmia achieves its peak performance with a sufficient reference-model budget.

\paragraph*{Pooled Variance for Stable Estimation under Limited Reference-Model Budgets}

As discussed in \secref{sec:method}, \plmia estimates the Gaussian parameters $\mu_{\mathrm{in}}, \sigma_{\mathrm{in}}$, $\mu_{\mathrm{out}},$ and $\sigma_{\mathrm{out}}$ in \eqref{eq:gaussian_model} to compute the GLR statistic.
When only a small number of reference models are available, the sample variance estimates can be unstable, thereby degrading attack power.
To mitigate this, we consider a \textit{pooled variance} that uses the information across augmented views of the same point to stabilize variance estimation.
As shown in \tabref{tab:num_ref_pooledaugment}, pooled variance is most beneficial when the reference-model budget is small.
The largest AUC gain occurs at $\#\mathrm{Ref}=4$, where pooled variance increases AUC from $61.71$ to $66.32$.
As the number of reference models increases, the standard variance estimates become sufficiently stable, and the difference between the two variants largely disappears.
This result highlights that incorporating augmentation-level information can partially compensate for limited model-level diversity in computation-constrained settings.

\begin{table*}[ht]
\centering
\small
\renewcommand{\tabcolsep}{0.3pc}
\renewcommand{\arraystretch}{0.5}
\caption{Effect of the reference-model budget on \plmia on \cifar{}, comparing the standard Gaussian variance estimator with the pooled variance estimator.}
\label{tab:num_ref_pooledaugment}
\begingroup
\begin{adjustbox}{width=\linewidth}
\begin{tabular}{llcccccccc}
\toprule
\multirow{2}{*}{Metric} & \multirow{2}{*}{\makecell[l]{Pooled\\variance}} & \multicolumn{8}{c}{\# of Reference Models} \\
\cmidrule(lr){3-10}
 &  & 2 & 4 & 8 & 16 & 32 & 64 & 128 & 254 \\
\midrule

\multirow{2}{*}{AUC}
& w/ & 65.32 & 66.32 & 67.84 & 69.68 & 70.87 & 71.68 & 72.20 & 72.48 \\
& w/o & 64.43 & 61.71 & 66.41 & 69.06 & 70.56 & 71.52 & 72.12 & 72.49 \\
\midrule
\multirow{2}{*}{TPR@0.01\%FPR}
& w/ & 0.71 & 0.98 & 1.12 & 2.00 & 3.19 & 3.88 & 4.46 & 4.40 \\
& w/o & 0.68 & 0.18 & 0.55 & 1.88 & 3.17 & 3.86 & 4.45 & 4.45 \\
\midrule
\multirow{2}{*}{TPR@0.0\%FPR}
& w/ & 0.35 & 0.61 & 0.72 & 1.23 & 2.54 & 2.97 & 3.22 & 3.28 \\
& w/o & 0.21 & 0.11 & 0.20 & 1.11 & 2.44 & 2.86 & 3.25 & 3.43 \\
\bottomrule
\end{tabular}
\end{adjustbox}
\endgroup
\end{table*}

\paragraph*{Number of Data Augmentations}
We then study the effect of the number of data augmentations on attack performance. As shown in \figref{fig:ablation-aug}, all \mias consistently exhibit performance gains across metrics as the augmentation budget increases.
\plmia consistently outperforms both \lira and \rmia across all considered augmentation budgets.
At TPR@0.0\%FPR, the performance gap becomes more pronounced with additional augmented views, indicating that \plmia effectively aggregates membership signals across multiple augmented views of the same point.
We further compare strategies for aggregating these augmented signals in Section~S2.2 of the Supplementary Material and show that mean aggregation yields the strongest overall performance.

\paragraph*{Number of Population Points}
\figref{fig:ablation-pop} shows the effect of the number of population points $|Z|$ used for pairwise comparison on the attack performance.
In terms of AUC, both \rmia and \plmia improve as $|Z|$ increases, indicating that a larger population set provides more reliable population calibration. \plmia consistently outperforms \rmia and gradually approaches the strong AUC performance of \lira as more population points become available.
The effect of $|Z|$ is more pronounced at critical low-FPR operating points.
With a limited number of population points, such as $|Z|=25$, the discrete voting mechanism used by \rmia becomes unstable and performs worse than \lira at both TPR@0.01\%FPR and TPR@0.0\%FPR.
In contrast, \plmia remains competitive with \lira in this low-$|Z|$ regime and surpasses both baselines once more population points are available.
These results support our theoretical insight that continuous $p$-value aggregation is more robust to limited population-calibration data than discrete voting-based aggregation.

\begin{figure}[!htbp]
    \centering
    \begin{subfigure}{0.7\linewidth}
        \centering
        \includegraphics[width=\linewidth]{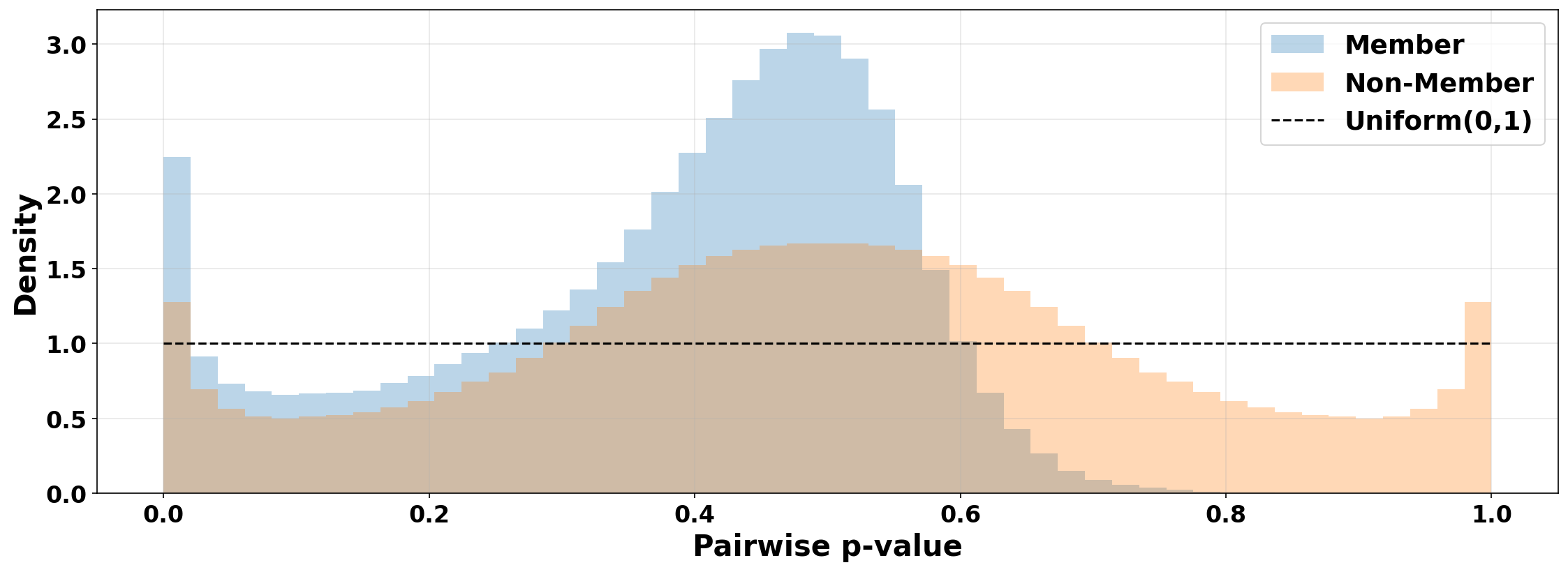}
        \caption{CIFAR-10}
        \label{fig:pvalue_cifar10}
    \end{subfigure}

    \begin{subfigure}{0.7\linewidth}
        \centering
        \includegraphics[width=\linewidth]{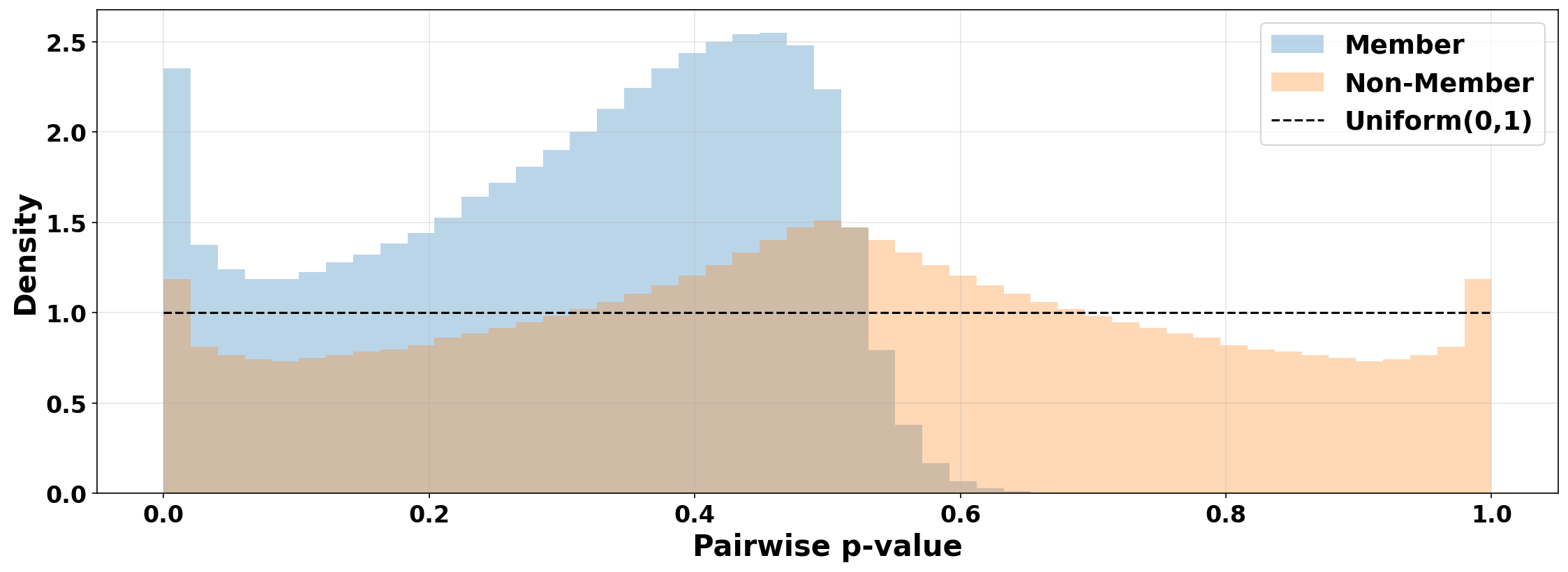}
        \caption{CIFAR-100}
        \label{fig:pvalue_cifar100}
    \end{subfigure}

    \begin{subfigure}{0.7\linewidth}
        \centering
        \includegraphics[width=\linewidth]{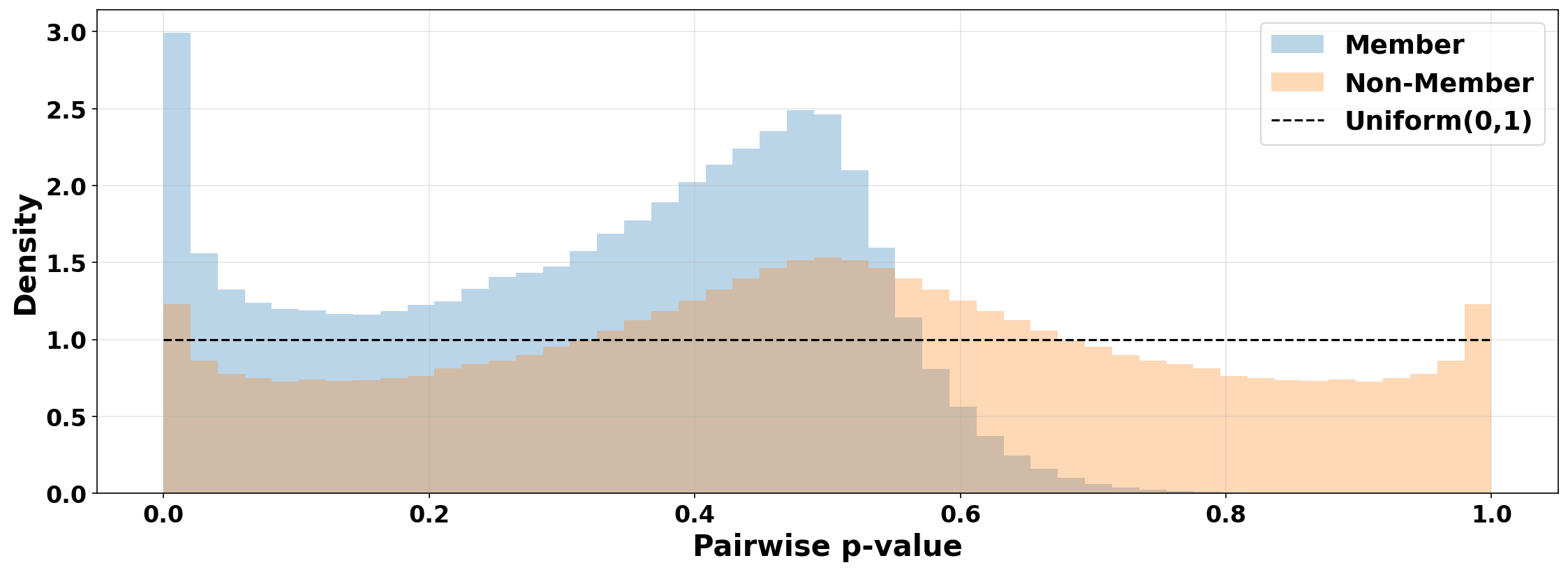}
        \caption{CINIC-10}
        \label{fig:pvalue_cinic10}
    \end{subfigure}

    \caption{Empirical distribution of pairwise $p$-values for member and non-member queries across datasets. The dashed line indicates the density of the uniform distribution $\mathrm{Uniform}(0,1)$, which serves as the idealized reference distribution under $H_0$.}
    \label{fig:pvalue_distribution}
\end{figure}

\subsubsection{Statistical Properties of \plmia}\label{sec:exp_property}

\paragraph*{Validity of Pairwise \texorpdfstring{$p$}{p}-values}
We examine the empirical distributions of pairwise $p$-values on \cifar, \cifarb, and \cinic.
As shown in \figref{fig:pvalue_distribution}, the pairwise $p$-values for non-members are relatively flat and close to the reference uniform density, with mild boundary deviations caused by finite-sample estimation and the Gaussian approximation.
By contrast, members show a clear departure from uniformity and assign more probability mass to small $p$-values.
These results support the validity of the proposed pairwise $p$-values as directional, continuous measures of membership evidence for the subsequent Cauchy combination.
The mild deviations from the idealized null motivate the finite-sample FPR analysis below.

\paragraph*{Empirical FPR Control with the Analytic Decision Threshold}

Under the ideal exact-Cauchy null, the Cauchy-combined score yields the analytic decision threshold
\begin{equation}
    \tau_\alpha = \tan\{\pi(0.5-\alpha)\},
\end{equation}
where $\alpha$ is the nominal FPR level. Exact agreement between the nominal and realized FPR requires that each pairwise $p$-value $p(q,z)$ follow $\mathrm{Uniform}(0,1)$ under $H_0$ and that the aggregated pairwise $p$-values be mutually independent. These conditions are only approximately satisfied in our setting.
For a fixed query point $q$, all pairwise $p$-values share the same query statistic
$\phi_{\mathrm{PL}}^{(q)}$ and are therefore dependent.
Moreover, \figref{fig:pvalue_distribution} shows that the $p$-values for non-members exhibit mild boundary deviations from uniformity.
The exact-Cauchy null should therefore be interpreted as an idealized reference distribution rather than the exact finite-sample null distribution.

\begin{figure}[!htbp]
    \centering
    \includegraphics[width=0.5\linewidth]{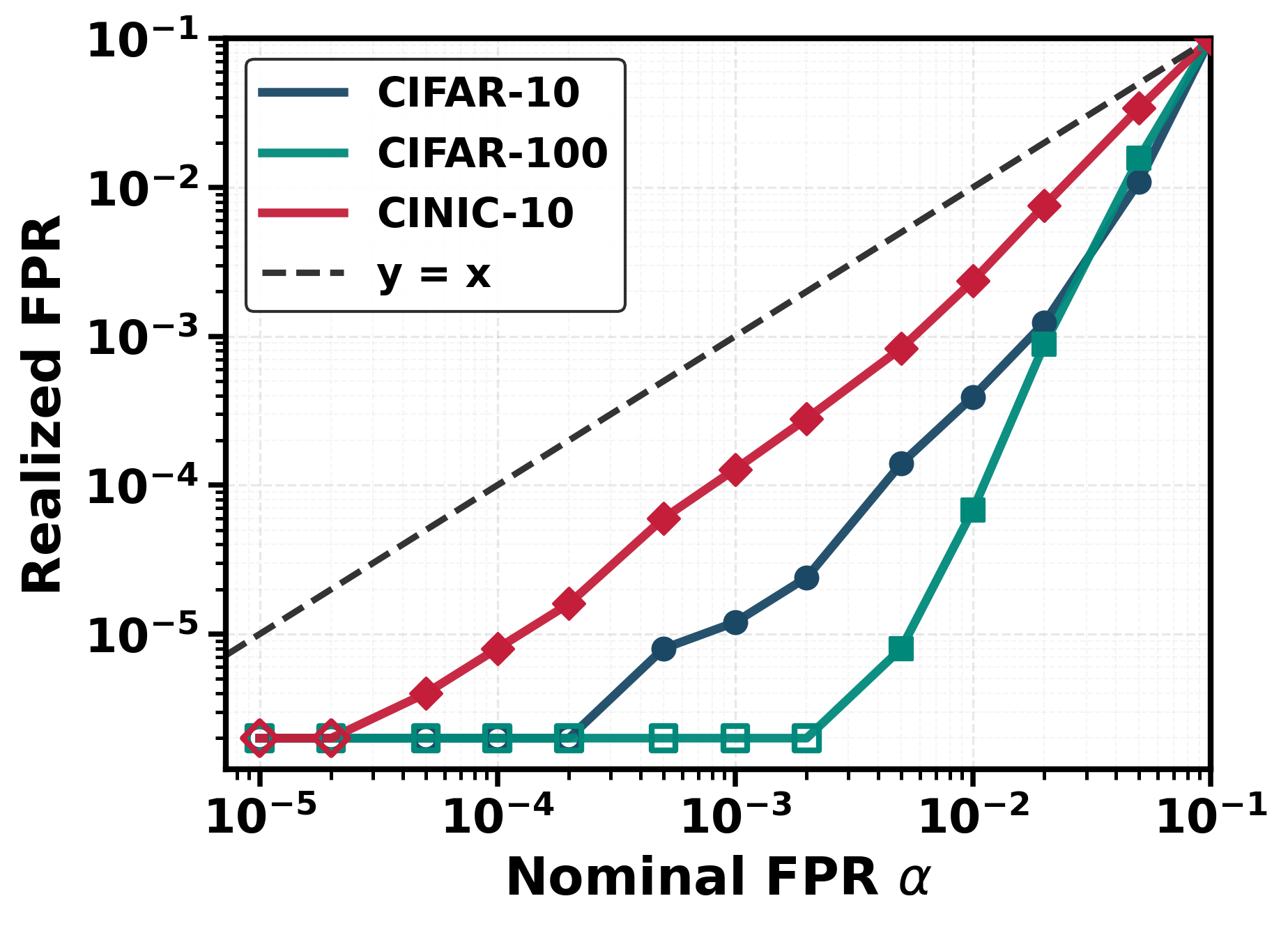}
    \caption{Nominal versus realized FPR under the analytic decision threshold. The dashed diagonal denotes exact agreement between the nominal and realized FPR.}
    \label{fig:cauchy_fpr_agreement}
\end{figure}

To assess the finite-sample behavior of the analytic threshold, we compare the nominal FPR with the empirically realized FPR in \figref{fig:cauchy_fpr_agreement}.
Let $Q_{\mathrm{out}}$ denote the non-member audit queries. For each nominal level $\alpha$, the realized FPR is computed as
\begin{equation}
    \widehat{\mathrm{FPR}}(\alpha) = \frac{1}{|Q_{\mathrm{out}}|} \sum_{q\in Q_{\mathrm{out}}}
    \mathbb{I}(\mathrm{Score}_{\mathrm{PL}}(q)\geq \tau_\alpha).
\end{equation}
The dashed diagonal line $y=x$ represents exact agreement, where the realized FPR equals the nominal level.
As shown in \figref{fig:cauchy_fpr_agreement}, the realized FPR remains below the nominal FPR for all three datasets.
Thus, the analytic decision threshold is conservative in our finite-sample experiments.
The deviation from exact agreement becomes more pronounced in the extreme low-FPR regime.
For example, at $\alpha=0.01\%$, the realized FPR is already below the nominal level, meaning that the analytic threshold flags fewer non-members than a finite-sample threshold that attains the nominal FPR exactly.
As $\alpha$ approaches zero, the ideal exact-Cauchy threshold diverges, whereas an empirical threshold yielding zero observed false positives remains finite for a finite non-member audit set.

Overall, the finite-sample departures from the Cauchy null occur in a conservative direction: the analytic threshold flags fewer non-members than the nominal FPR permits.
This behavior is consistent with the tail-validity result of \cite{cct}, which supports Cauchy-combination tail control under dependence.
Although the finite-sample pairwise $p$-values are not exactly uniform and the transformed terms are dependent, the empirical results show that the analytic threshold provides conservative low-FPR control in the evaluated settings.

\subsubsection{Generalization to Tabular Data and Different Model Families}
\label{sec:exp_tabular_data_tree}

We further evaluate \plmia's generalizability to tabular data and different model families.
Specifically, we consider Gradient Boosting Decision Tree (GBDT) models~\citep{gbdt} on the tabular Purchase-100 dataset to examine whether \plmia extends beyond image data and neural networks to tabular data and tree-based models.
We train GBDT models with $250$ estimators, a learning rate of $0.1$, and a subsampling rate of $0.2$.
We vary the maximum tree depth over ${3,5,7}$ and use the true-label prediction probability as the attack signal.
For each setting, all attacks follow the same member/non-member evaluation protocol and use the same sets of target and reference models.

\begin{table}[htbp]
\centering
\small
\renewcommand{\tabcolsep}{0.3pc}
\renewcommand{\arraystretch}{0.5}
\caption{Attack performance on Purchase-100 using GBDT models with different maximum depths. The maximum depth controls the complexity of the tree-based statistical model. Results are averaged over target models and reported as mean $\pm$ standard deviation.}
\label{tab:purchase_gbdt}
\begin{tabular}{c|l|ccc}
\toprule
\texttt{max\_depth} & Method & AUC & TPR@0.01\%FPR & TPR@0\%FPR \\
\midrule
\multirow{3}{*}{3}
 & PL-MIA & $\mathbf{94.36 \pm 0.08}$ & $\mathbf{6.20 \pm 2.01}$ & $\mathbf{3.31 \pm 1.65}$ \\
 & LiRA   & $93.79 \pm 0.10$ & $1.34 \pm 0.68$ & $0.42 \pm 0.38$ \\
 & RMIA   & $94.23 \pm 0.08$ & $4.66 \pm 1.35$ & $2.45 \pm 1.44$ \\
\midrule
\multirow{3}{*}{5}
 & PL-MIA & $\mathbf{98.71 \pm 0.03}$ & $\mathbf{23.62 \pm 4.75}$ & $\mathbf{15.53 \pm 5.71}$ \\
 & LiRA   & $98.47 \pm 0.04$ & $3.86 \pm 1.79$ & $1.29 \pm 1.28$ \\
 & RMIA   & $98.61 \pm 0.03$ & $19.66 \pm 3.06$ & $12.48 \pm 4.80$ \\
\midrule
\multirow{3}{*}{7}
 & PL-MIA & $\mathbf{99.14 \pm 0.02}$ & $\mathbf{31.25 \pm 1.74}$ & $\mathbf{21.86 \pm 5.90}$ \\
 & LiRA   & $98.93 \pm 0.02$ & $3.83 \pm 1.94$ & $0.81 \pm 0.64$ \\
 & RMIA   & $99.04 \pm 0.03$ & $23.11 \pm 2.40$ & $12.41 \pm 5.92$ \\
\bottomrule
\end{tabular}
\end{table}

\tabref{tab:purchase_gbdt} and \figref{fig:purchase100_gbdt} show that \plmia remains effective across GBDT models of different depths.
Attack performance improves substantially as \texttt{max\_depth} grows, indicating that deeper GBDT models expose more distinguishable membership evidence.
For all depths, \plmia achieves the highest AUC and the strongest performance under stringent FPR constraints among the compared attacks.
The improvement is especially pronounced at TPR@0\%FPR: when \texttt{max\_depth}=7, \plmia reaches $21.86$, exceeding \rmia{}'s $12.41$ by $76.15\%$.
We also evaluate Multi-Layer Perceptron (MLP) models on Purchase-100, with the results reported in Section~S2.4 of the Supplementary Material.
Together, these experiments demonstrate that \plmia generalizes across data modalities and model families, from image data with neural networks to tabular data with both neural and tree-based models.

\begin{figure}[!htb]
    \centering
    \includegraphics[width=0.9\linewidth]{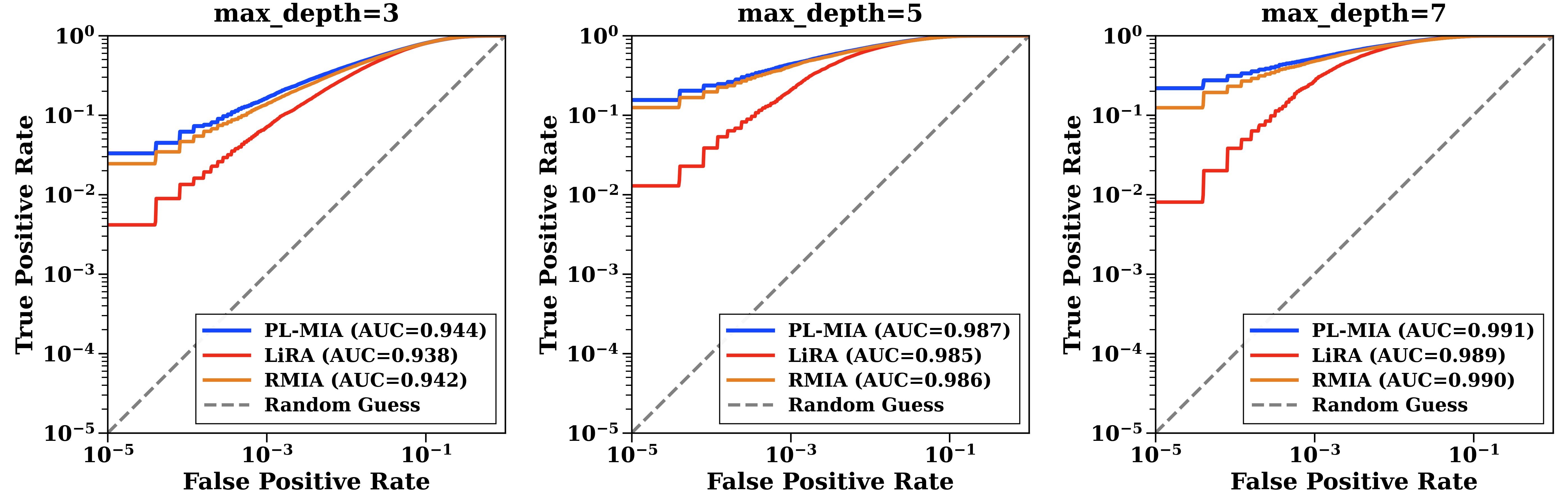}
    \caption{ROC curves on Purchase-100 with GBDT models. Each panel corresponds to one maximum-depth setting and compares \plmia against attack baselines.}
    \label{fig:purchase100_gbdt}
\end{figure}

\section{Conclusion}\label{sec:conclusion}

Our findings indicate that the low-FPR gains of PL-MIA arise from the interaction of three statistical mechanisms.
First, the Gaussian likelihood ratio uses both the mean shift and the variance contraction between the IN and OUT distributions. The latter represents an additional membership signal because models tend to exhibit lower predictive uncertainty on training members.
Second, population calibration treats query-specific heterogeneity as a nuisance effect. Comparing a query point with held-out non-members places its statistic on a relative scale, reducing the likelihood that an intrinsically easy non-member receives a high membership score.
Third, continuous $p$-values and the Cauchy combination preserve more information than binary voting. Individual pairwise comparisons may be weak or noisy, but combination testing can synthesize them into an overall assessment while allowing a small number of highly informative comparisons to contribute substantial evidence~\citep{liu2019acat,cct}.
PL-MIA primarily improves separation in the tail of the non-member distribution, which is the region most relevant to reliable privacy auditing.

Several extensions could strengthen the statistical foundation and practical scope of PL-MIA.
An important problem is the optimal allocation of a fixed computational budget among reference models, population points, data augmentations, and audit queries.
For large-scale models, independently training many reference models may be infeasible.
Training checkpoints, smaller proxy models, or token- and representation-level reference distributions may provide alternative sources of calibration information~\citep{informia,hayes2025exploring}.
Together, these directions could extend the proposed framework from benchmark attack comparison toward reliable, uncertainty-aware, and reproducible privacy auditing.

{
}

\bigskip
 \begin{center}
 \textbf{\large Appendices: Supplementary Materials}
 \end{center}
 
The supplementary materials are organized as follows.
\appref{app:setup} provides additional details on the model-training protocol and hyperparameters for the image-dataset experiments.
\appref{app:result} presents supplementary experiments on limited reference-model budgets, augmentation aggregation strategies, the Gaussian assumption, and MLP models on Purchase-100.
\appref{app:proof} presents detailed proofs of the theoretical results in Section~4 of the main manuscript.

 \appendix
\section{Model Training Protocol and Hyperparameters}
\label{app:setup}

We describe the overall experimental setup, including the datasets, attack modes, baselines, and evaluation metrics in  Section~5.1 of the main manuscript.
Here, we provide additional details on the model-training protocol and hyperparameters for the image-dataset experiments.
Following prior work~\citep{lira,rmia}, each target and reference model is trained on a $50\%$ subset of the full dataset.
For each target model, the training subset defines the member set, while the remaining half is held out as non-members.
Each reference model is trained on an independently sampled $50\%$ subset, so that each query point is included in approximately half of the reference models, yielding approximately balanced IN and OUT reference sets.

For CIFAR-10, CIFAR-100, and CINIC-10, we use Wide-ResNet as the backbone for both target and reference models.
All models are trained using stochastic gradient descent (SGD) with cosine learning-rate decay and a brief linear warm-up over the first $1\%$ of the training epochs.
We do not apply data augmentation during model training.
\tabref{tab:training_protocol} summarizes the training hyperparameters and the average training/test accuracy across target models.

\begin{table}[htbp]
\centering
\caption{Training hyperparameters and average training/test accuracy (\%) for the image-dataset experiments.}
\label{tab:training_protocol}
\resizebox{\linewidth}{!}{%
\begin{tabular}{lccccccc}
\toprule
Dataset & Epoch & Batch size &
Learning rate & Momentum &
Weight decay & Train Acc. & Test Acc.\\
\midrule
\cifar  & 100 & 256 & 0.1 & 0.9 & $5\times10^{-4}$ & 100.0 & 91.7 \\
\cifarb & 100 & 256 & 0.1 & 0.9 & $5\times10^{-4}$ & 99.9  & 67.1 \\
\cinic  & 100 & 256 & 0.1 & 0.9 & $5\times10^{-4}$ & 99.7  & 78.0 \\
\bottomrule
\end{tabular}
}
\end{table}

\section{Supplementary Experimental Results}\label{app:result}

This section presents additional experimental results related to the analyses in Section~5.3 of the main manuscript.

\subsection{Performance under Limited Reference-Model Budgets}

\begin{table}[!htbp]
\centering
\small
\renewcommand{\tabcolsep}{0.3pc}
\renewcommand{\arraystretch}{0.5}
\caption{Attack performance on CIFAR-100 under different reference-model budgets. Here, \# Ref denotes the total number of reference models; \# Ref $=4$ corresponds to two IN and two OUT reference models per query point. ``PL-MIA (pool var.)'' denotes the pooled-variance variant of PL-MIA, which uses a pooled variance estimate to stabilize Gaussian modeling when only a small number of reference models is available. Results are averaged over target models and reported as mean $\pm$ standard deviation.}
\label{tab:cifar100_num_ref_results}
\begin{tabular}{c l|ccc}
\toprule
\# Ref & Attack & AUC & \multicolumn{2}{c}{TPR@FPR} \\
\cmidrule(lr){4-5}
 &  &  & 0.01\% & 0.0\% \\
\midrule
\multirow{4}{*}{4}
 & LiRA              & $87.74 \pm 0.12$ & $3.71 \pm 1.17$ & $1.49 \pm 0.94$ \\
 & RMIA              & $\mathbf{89.54 \pm 0.13}$ & $\mathbf{3.84 \pm 1.76}$ & $\mathbf{2.15 \pm 1.74}$ \\
 & PL-MIA            & $67.55 \pm 0.15$ & $0.01 \pm 0.01$ & $0.00 \pm 0.01$ \\
 & PL-MIA (pool var.) & $86.03 \pm 0.14$ & $0.68 \pm 0.27$ & $0.21 \pm 0.23$ \\
\midrule
\multirow{4}{*}{8}
 & LiRA              & $89.27 \pm 0.13$ & $6.07 \pm 1.44$ & $2.88 \pm 1.65$ \\
 & RMIA              & $\mathbf{90.19 \pm 0.15}$ & $\mathbf{6.97 \pm 1.36}$ & $\mathbf{4.28 \pm 2.52}$ \\
 & PL-MIA            & $85.76 \pm 0.20$ & $0.12 \pm 0.09$ & $0.06 \pm 0.06$ \\
 & PL-MIA (pool var.) & $88.92 \pm 0.14$ & $1.96 \pm 0.84$ & $0.73 \pm 0.32$ \\
\midrule
\multirow{4}{*}{16}
 & LiRA              & $90.29 \pm 0.14$ & $7.41 \pm 1.80$ & $\mathbf{4.39 \pm 1.90}$ \\
 & RMIA              & $\mathbf{90.55 \pm 0.14}$ & $\mathbf{10.82 \pm 1.10}$ & $4.38 \pm 3.69$ \\
 & PL-MIA            & $90.06 \pm 0.17$ & $4.32 \pm 0.86$ & $2.16 \pm 1.17$ \\
 & PL-MIA (pool var.) & $90.49 \pm 0.15$ & $6.00 \pm 1.82$ & $3.16 \pm 1.88$ \\
\midrule
\multirow{4}{*}{32}
 & LiRA              & $90.90 \pm 0.10$ & $9.91 \pm 2.26$ & $6.07 \pm 3.40$ \\
 & RMIA              & $90.78 \pm 0.13$ & $\mathbf{11.50 \pm 2.21}$ & $5.63 \pm 5.36$ \\
 & PL-MIA            & $91.24 \pm 0.14$ & $10.29 \pm 2.45$ & $6.18 \pm 2.50$ \\
 & PL-MIA (pool var.) & $\mathbf{91.33 \pm 0.13}$ & $11.12 \pm 2.41$ & $\mathbf{6.84 \pm 2.33}$ \\
\midrule
\multirow{4}{*}{64}
 & LiRA              & $91.22 \pm 0.12$ & $10.67 \pm 1.98$ & $7.17 \pm 3.63$ \\
 & RMIA              & $90.89 \pm 0.14$ & $11.80 \pm 3.06$ & $6.61 \pm 5.53$ \\
 & PL-MIA            & $91.73 \pm 0.14$ & $13.15 \pm 2.89$ & $10.09 \pm 2.95$ \\
 & PL-MIA (pool var.) & $\mathbf{91.75 \pm 0.14}$ & $\mathbf{13.43 \pm 2.69}$ & $\mathbf{10.35 \pm 2.49}$ \\
\bottomrule
\end{tabular}
\end{table}

To complement the cost analysis in Section~5.3.2 of the main manuscript, we further compare attacks on CIFAR-100 under limited reference-model budgets.
We start with $\#\mathrm{Ref}=4$, which corresponds to approximately two IN and two OUT reference models per query point.
We do not consider $\#\mathrm{Ref}=2$, since this leaves only one IN and one OUT reference model per query point on average, precluding reliable estimation of the corresponding distributions.

\tabref{tab:cifar100_num_ref_results} reveals a clear budget-dependent trade-off.
When the reference-model budget is extremely limited ($\#\mathrm{Ref}\in\{4,8,16\}$), \rmia achieves the strongest low-FPR performance, whereas the standard \plmia is more sensitive to unstable Gaussian variance estimates.
Pooling variance across augmented views substantially stabilizes \plmia in this regime. The largest gain occurs at $\#\mathrm{Ref}=4$, where the pooled variance increases AUC from $67.55$ to $86.03$.

As the reference-model budget increases, \plmia rapidly closes the gap.
At \#Ref $=32$, \plmia surpasses \rmia in AUC and TPR@0.0\%FPR while remaining competitive at TPR@0.01\%FPR.
At \#Ref $=64$, both the standard and pooled-variance variants outperform \rmia across all reported metrics.
These results indicate that \rmia is preferable under extremely limited reference-model budgets, whereas \plmia becomes more effective once the Gaussian parameters can be estimated with sufficient stability. Pooled variance partially mitigates this requirement in low-budget settings.

\subsection{Effect of Augmentation Aggregation Strategies}

In addition to the augmentation-budget analysis in Section~5.3.2 of the main manuscript, we further investigate how membership signals should be aggregated across multiple augmented views.
Given $K$ augmented views for each query $q$ and a population point $z$, we compute augmentation-specific pairwise differences and aggregate them into a single pairwise signal before converting it into membership evidence.
We compare three aggregation strategies: \emph{Mean}, \emph{Median}, and \emph{Voting}~\citep{rmia}, where \emph{Voting} records the fraction of augmented comparisons for which $q$ appears more member-like than $z$.

\begin{table}[!htbp]
\centering
\small
\renewcommand{\tabcolsep}{0.3pc}
\renewcommand{\arraystretch}{0.5}
\caption{Effect of augmentation aggregation strategies on attack performance. \emph{Mean} and \emph{Median} aggregate the augmentation-specific pairwise differences, whereas \emph{Voting} records the fraction of positive pairwise comparisons.}
\label{tab:aug_agg_cifar10}
\begin{tabular}{l|ccc}
\toprule
Aggregation & AUC & \multicolumn{2}{c}{TPR@FPR} \\
\cmidrule(lr){3-4}
 &  & 0.01\% & 0.0\% \\
\midrule
Voting & 71.71 & \textbf{4.49} & 3.06 \\
Median & 72.39 & 4.39 & 3.38 \\
Mean   & \textbf{72.49} & 4.45 & \textbf{3.43} \\
\bottomrule
\end{tabular}
\end{table}

As shown in \tabref{tab:aug_agg_cifar10}, \emph{Mean} aggregation provides the strongest overall performance. It achieves the highest AUC and TPR@0\%FPR, while \emph{Voting} is only marginally higher at TPR@0.01\%FPR ($4.49$ versus $4.45$). \emph{Median} aggregation remains competitive but does not improve the overall trade-off.
These results indicate that retaining the magnitude of augmentation-specific signals through averaging is generally more effective than using a robust summary or discretizing each comparison into a binary vote. We therefore use mean aggregation throughout the experiments.

\subsection{Gaussian Assumption Diagnostics and Robustness}

\plmia models the per-point IN and OUT reference distributions using Gaussian approximations, as specified in Assumption~2.2 of the main manuscript.
Although this assumption makes the likelihood-ratio statistic analytically tractable, it may not hold exactly for every query point.
We therefore examine both the empirical plausibility of the Gaussian approximation and \plmia's robustness to departures from Gaussianity.

For each query point $q$, we collect the logit-scaled true-label confidence scores from reference models that include $q$ in their training sets and from those that exclude it, forming the empirical IN and OUT distributions, respectively.
Since our experiments use 254 reference models, each distribution contains 127 observations.
We apply the Shapiro-Wilk normality test~\citep{sw_test} separately to the IN and OUT distributions.
For observations $r_1,\ldots,r_n$, the Shapiro-Wilk statistic is
\begin{equation}
    W = \frac{\left(\sum_{i=1}^{n} \omega_i r_{(i)}\right)^2}{\sum_{i=1}^{n}(r_i-\frac{1}{n}\sum_{i=1}^n r_i)^2},
\end{equation}
where $r_{(i)}$ denotes the $i$-th order statistic and $\omega_i$ denotes the corresponding Shapiro-Wilk weights.
A value of $W$ closer to one indicates stronger agreement with a Gaussian distribution, while a small associated $p$-value provides evidence against the normality hypothesis.

\begin{figure}[!htbp]
    \centering
    \includegraphics[width=\linewidth]{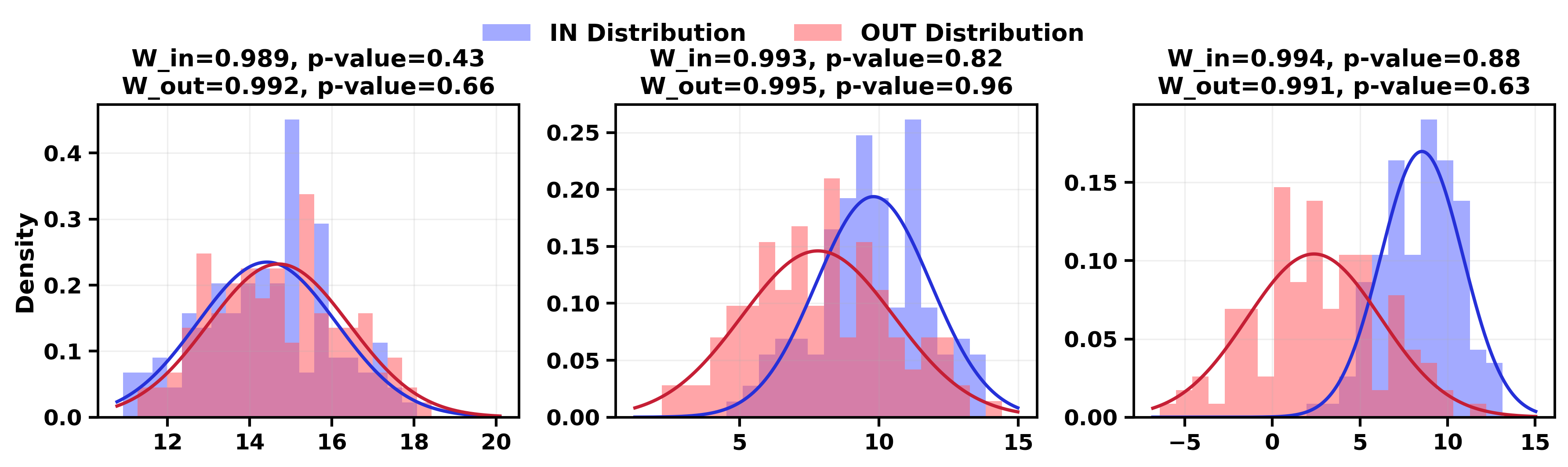}
    \caption{Representative per-point normality diagnostics on CIFAR-10. Each panel shows the empirical IN and OUT distributions of logit-scaled true-label confidence scores, with fitted Gaussian densities overlaid. The annotations report the Shapiro--Wilk statistics and corresponding $p$-values for the two reference distributions.}
    \label{fig:per_example_gaussian}
\end{figure}

As illustrated in \figref{fig:per_example_gaussian}, the empirical histograms are generally aligned with the fitted Gaussian densities, supporting the use of the Gaussian approximation as a practical model for the reference distributions.
However, the degree of agreement varies across query points. To quantify the effect of this variation, we define the Gaussianity score as
\begin{equation}
    G(q)=\min\{W_{\mathrm{in}}(q), W_{\mathrm{out}}(q)\},
\end{equation}
where $W_{\mathrm{IN}}(q)$ and $W_{\mathrm{OUT}}(q)$ are the Shapiro-Wilk statistics for the IN and OUT distributions of $q$, respectively.
This score is conservative: a query point receives a high Gaussianity score only when both its IN and OUT distributions are close to Gaussian.
We sort the audited query points by $G(q)$ and partition them into three equal-sized groups: the most Gaussian third, the middle third, and the least Gaussian third.
We then evaluate \plmia separately on each group.

\begin{table}[!htbp]
\centering
\small
\renewcommand{\tabcolsep}{0.3pc}
\renewcommand{\arraystretch}{0.5}
\caption{\plmia performance under different degrees of Gaussianity. Query points are partitioned into three equal-sized groups by the Gaussianity score $G(q)=\min{W_{\mathrm{in}}(q), W_{\mathrm{out}}(q)}$, where $W_{\mathrm{in}}$ and $W_{\mathrm{out}}$ are the Shapiro-Wilk statistics of the IN and OUT reference distributions. Results are averaged over target models and reported as mean $\pm$ standard deviation.}
\label{tab:plmia_gaussian_robust}
\begin{tabular}{l|ccc}
\toprule
Group & AUC & TPR@1\%FPR & TPR@0.1\%FPR \\
\midrule
Most Gaussian  & $\mathbf{73.03 \pm 0.58}$ & $\mathbf{17.48 \pm 0.66}$ & $\mathbf{9.62 \pm 0.75}$ \\
Middle third   & $72.73 \pm 0.40$ & $16.14 \pm 0.80$ & $7.94 \pm 0.99$ \\
Least Gaussian & $71.68 \pm 0.31$ & $14.19 \pm 0.60$ & $7.45 \pm 0.65$ \\
\bottomrule
\end{tabular}
\end{table}

\tabref{tab:plmia_gaussian_robust} shows that attack performance degrades gradually rather than collapsing as the reference distributions become less Gaussian.
On the most Gaussian third, \plmia achieves the best performance, with an AUC of $73.03$, TPR@1\%FPR of $17.48$, and TPR@0.1\%FPR of $9.62$.
The performance degradation is mild across the three groups.
Even on the least Gaussian third, \plmia still achieves an AUC of $71.68$, a TPR@1\%FPR of $14.19$, and a TPR@0.1\%FPR of $7.45$.
Thus, moving from the most Gaussian group to the least Gaussian group reduces AUC by only $1.35$ percentage points, while the low-FPR TPR remains meaningful.

These results suggest that moderate departures from Gaussianity do not prevent \plmia from extracting useful membership signal.
The Gaussian approximation is used locally to estimate the pointwise IN/OUT likelihood ratio, whereas the final score aggregates evidence over many population comparisons; this aggregation may reduce the impact of local distributional misspecification.
We therefore view Assumption~2.2 of the main manuscript as a practical modeling approximation rather than a strict empirical requirement.
Replacing the Gaussian estimator with more flexible parametric or nonparametric density estimators is a natural direction for future work.

\subsection{Multi-Layer Perceptron on Purchase-100}

To test whether \plmia generalizes to tabular data while the target model remains a neural network, we further evaluate \plmia on Purchase-100 using a Multi-Layer Perceptron (MLP)~\citep{mlp}.
This experiment complements the GBDT results in Section~5.3.4 of the main manuscript, which evaluates \plmia on a tree-based model.
We use a four-layer MLP with hidden-layer sizes $[512, 256, 128, 64]$, trained on 25k samples for 50 epochs.
All attacks follow the same member/non-member evaluation protocol and use the same sets of target and reference models as in the main experiments.

\begin{table}[!htbp]
\centering
\small
\renewcommand{\tabcolsep}{0.3pc}
\renewcommand{\arraystretch}{0.5}
\caption{Attack performance on Purchase-100 using the four-layer MLP model. Results are averaged over target models and reported with standard deviations.}
\label{tab:purchase_mlp}
\begin{tabular}{l|ccc}
\toprule
Method & AUC & TPR@0.01\%FPR & TPR@0\%FPR \\
\midrule
PL-MIA & $\mathbf{84.23 \pm 0.33}$ & $\mathbf{3.70 \pm 1.01}$ & $\mathbf{2.07 \pm 0.90}$ \\
LiRA   & $83.23 \pm 0.30$ & $1.95 \pm 0.36$ & $0.83 \pm 0.28$ \\
RMIA   & $83.71 \pm 0.30$ & $3.20 \pm 0.56$ & $1.59 \pm 0.64$ \\
\bottomrule
\end{tabular}
\end{table}

\begin{figure}[htb]
    \centering
    \includegraphics[width=0.4\linewidth]{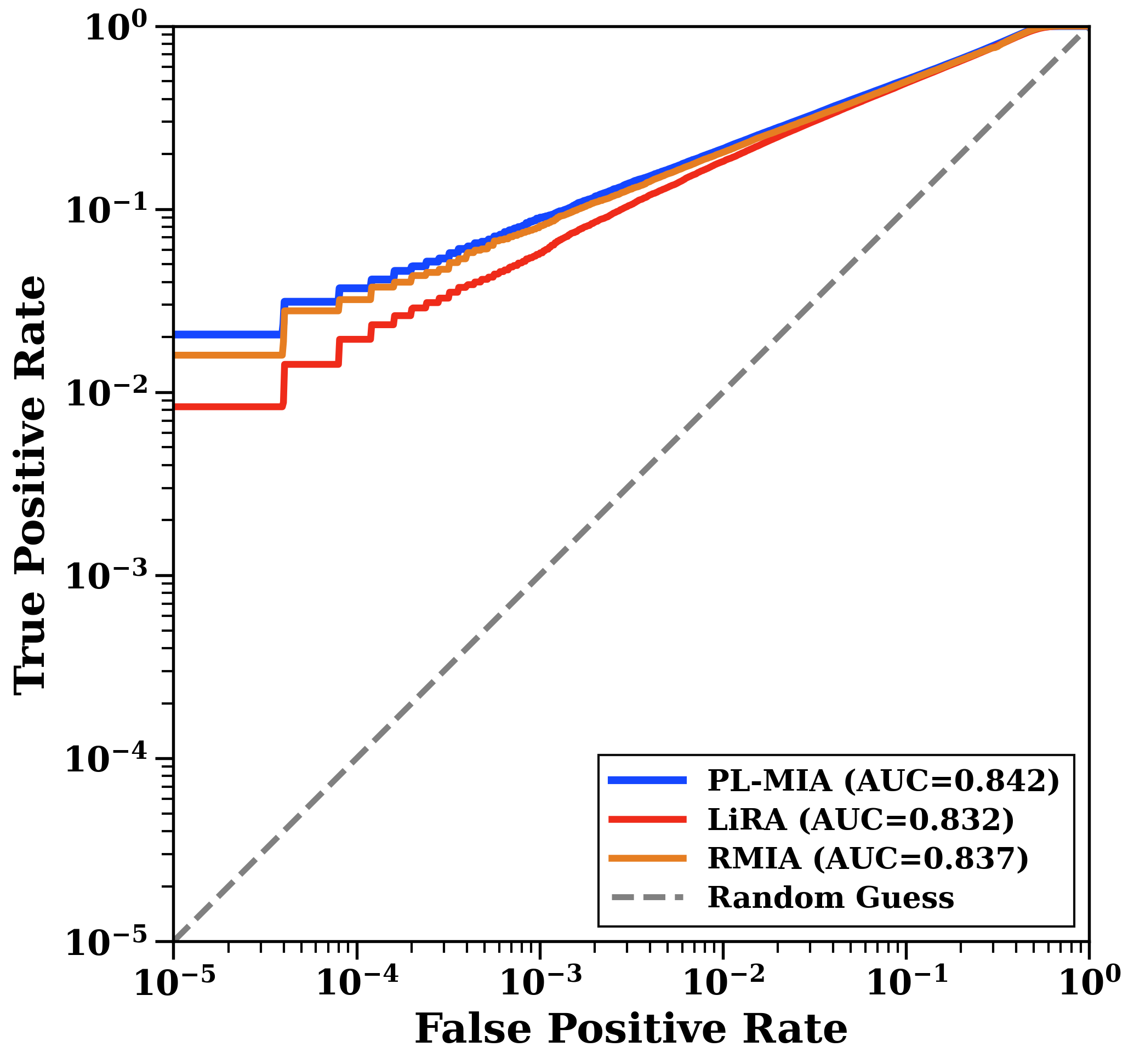}
    \caption{ROC curves on Purchase-100 with the four-layer MLP model. The comparison evaluates \plmia against baselines under the tabular neural-network setting.}
    \label{fig:purchase100_mlp}
\end{figure}

As shown in \tabref{tab:purchase_mlp} and \figref{fig:purchase100_mlp}, \plmia achieves the strongest performance among the compared attacks.
It obtains an AUC of $84.23$, compared with $83.23$ for \lira and $83.71$ for \rmia.
The advantage becomes more pronounced in the low-FPR regime: \plmia achieves $3.70$ TPR@0.01\%FPR and $2.07$ TPR@0\%FPR, improving over \rmia{}'s $3.20$ and $1.59$, respectively.
These results further demonstrate that \plmia remains effective on tabular data even when the target model is a neural network.

\section{Proofs}\label{app:proof}

This section provides detailed proofs of the theoretical results in Section~4 of the main manuscript. We first derive the GLR and BLR statistics (Lemmas~4.1 and 4.2), then establish the pairwise-difference distributions and the neutral baselines induced by population calibration (Lemmas~4.3 and 4.4). Finally, we derive the attack-power expressions and prove the resulting theoretical ordering (Theorems~4.5 and 4.6).

\subsection{Proof of Lemma~4.1}

\begin{proof}
For a Gaussian distribution $\mathcal{N}(\mu, \sigma^2)$, the log-density is:
\begin{align}
    \log \mathrm{pdf}(x \mid \mathcal{N}(\mu, \sigma^2)) = -\frac{1}{2}\log(2\pi) - \log\sigma - \frac{(x-\mu)^2}{2\sigma^2}.
\end{align}
It follows that the Gaussian LR (GLR) is
\begin{align}
    \varphi_{\mathrm{GLR}}^{(d)} &= \log \mathrm{pdf}\big(\mathrm{conf}^{(d)} \mid \mathcal{N}(\mu_{\mathrm{in}}^{(d)},(\sigma_{\mathrm{in}}^{(d)})^2)\big) - \log \mathrm{pdf}\big(\mathrm{conf}^{(d)} \mid \mathcal{N}(\mu_{\mathrm{out}}^{(d)},(\sigma_{\mathrm{out}}^{(d)})^2)\big) \nonumber\\
    &= \left(-\log\sigma_{\mathrm{in}}^{(d)} - \frac{(\mathrm{conf}^{(d)} - \mu_{\mathrm{in}}^{(d)})^2}{2(\sigma_{\mathrm{in}}^{(d)})^2}\right) - \left(-\log\sigma_{\mathrm{out}}^{(d)} - \frac{(\mathrm{conf}^{(d)} - \mu_{\mathrm{out}}^{(d)})^2}{2(\sigma_{\mathrm{out}}^{(d)})^2}\right) \nonumber\\
    &= \log\frac{\sigma_{\mathrm{out}}^{(d)}}{\sigma_{\mathrm{in}}^{(d)}} + \frac{(\mathrm{conf}^{(d)} - \mu_{\mathrm{out}}^{(d)})^2}{2(\sigma_{\mathrm{out}}^{(d)})^2} - \frac{(\mathrm{conf}^{(d)} - \mu_{\mathrm{in}}^{(d)})^2}{2(\sigma_{\mathrm{in}}^{(d)})^2}. \label{eq:glr_general_form}
\end{align}

Using
$\sigma_{\mathrm{in}}^{(d)}=\rho\sigma_{\mathrm{out}}^{(d)}$ and $\eta^{(d)}=(\mu_{\mathrm{in}}^{(d)}-\mu_{\mathrm{out}}^{(d)})/\sigma_{\mathrm{out}}^{(d)}$, we obtain the two forms stated in Lemma~4.1. Let $W\sim\mathcal{N}(0,1)$ denote a standard normal variable.

\textit{Non-member case.} Under $H_0$, write
$\mathrm{conf}^{(d)}=\mu_{\mathrm{out}}^{(d)}+\sigma_{\mathrm{out}}^{(d)}W$. Substitution into \eqref{eq:glr_general_form} gives
\begin{align}
    \varphi_{\mathrm{GLR}}^{(d)} \mid H_0
    &= -\log\rho + \frac{W^2}{2} - \frac{(W-\eta^{(d)})^2}{2\rho^2}. \label{eq:glr_h0_form}
\end{align}

\textit{Member case.} Under $H_1$, write
$\mathrm{conf}^{(d)}=\mu_{\mathrm{in}}^{(d)}+\sigma_{\mathrm{in}}^{(d)}W=\mu_{\mathrm{out}}^{(d)}+\sigma_{\mathrm{out}}^{(d)}(\eta^{(d)}+\rho W)$. Substitution into \eqref{eq:glr_general_form} gives
\begin{align}
    \varphi_{\mathrm{GLR}}^{(d)} \mid H_1
    &= -\log\rho + \frac{(\eta^{(d)}+\rho W)^2}{2} - \frac{W^2}{2}. \label{eq:glr_h1_form}
\end{align}

We next compute the mean separation of GLR for $\rho \leq 1$. From \eqref{eq:glr_h1_form},
\begin{align}
    \mathbb{E}[\varphi_{\mathrm{GLR}}^{(d)} \mid H_1]
    &= -\log\rho + \frac{\mathbb{E}[(\eta^{(d)}+\rho W)^2]}{2} - \frac{\mathbb{E}[W^2]}{2} \nonumber\\
    &= -\log\rho + \frac{\rho^2 + (\eta^{(d)})^2 - 1}{2}. \label{eq:glr_mean_h1}
\end{align}
From \eqref{eq:glr_h0_form},
\begin{align}
    \mathbb{E}[\varphi_{\mathrm{GLR}}^{(d)} \mid H_0]
    &= -\log\rho + \frac{\mathbb{E}[W^2]}{2} - \frac{\mathbb{E}[(W-\eta^{(d)})^2]}{2\rho^2} \nonumber\\
    &= -\log\rho + \frac{1}{2} - \frac{1 + (\eta^{(d)})^2}{2\rho^2}. \label{eq:glr_mean_h0}
\end{align}
Subtracting the two expectations gives
\begin{align}
    \Delta\mu_{\mathrm{GLR}}^{(d)}
    &= \mathbb{E}[\varphi_{\mathrm{GLR}}^{(d)} \mid H_1] - \mathbb{E}[\varphi_{\mathrm{GLR}}^{(d)} \mid H_0] \nonumber\\
    &= \left(\frac{\rho^2 + (\eta^{(d)})^2 - 1}{2}\right) - \left(\frac{1}{2} - \frac{1 + (\eta^{(d)})^2}{2\rho^2}\right) \nonumber\\
    &= \frac{(\rho^2-1)^2 + (\eta^{(d)})^2(1+\rho^2)}{2\rho^2}. \label{eq:delta_mu_lambda}
\end{align}
When $\rho<1$, the separation exceeds $(\eta^{(d)})^2$ and grows as $O(1/\rho^2)$ as $\rho$ decreases. Thus, GLR captures variance contraction as an additional membership signal.

For the pairwise analysis in Lemma~4.3, setting $\rho=1$ in \eqref{eq:glr_h0_form} and \eqref{eq:glr_h1_form} gives
\begin{align}
    \varphi_{\mathrm{GLR}}^{(d)} \mid H_1
    &= \eta^{(d)}W + \frac{(\eta^{(d)})^2}{2}
    \sim \mathcal{N}\left(\frac{(\eta^{(d)})^2}{2}, (\eta^{(d)})^2\right), \\
    \varphi_{\mathrm{GLR}}^{(d)} \mid H_0
    &= \eta^{(d)}W - \frac{(\eta^{(d)})^2}{2}
    \sim \mathcal{N}\left(-\frac{(\eta^{(d)})^2}{2}, (\eta^{(d)})^2\right).
\end{align}
These Gaussian forms provide the GLR inputs to the pairwise-difference derivation.
\end{proof}

\subsection{Proof of Lemma~4.2}

\begin{proof}
Recall that the BLR pointwise statistic is
$\varphi_{\mathrm{BLR}}^{(d)}
:=\log(f_{\theta_t}(x_d)[y_d]/\bar c^{(d)})$, where
$\bar c^{(d)}=\mathbb{E}[f_{\theta}(x_d)[y_d]]$ is the average TLC over all reference models.
Let $h(u)=\operatorname{sigmoid}(u)$,
$g(u)=\log h(u)$, and
$s^{(d)}=h(\mu_{\mathrm{out}}^{(d)})$.
Under $H_0$ and $H_1$, respectively,
\begin{align*}
    f_{\theta_t}(x_d)[y_d]\mid H_0
    &=h\!\left(\mu_{\mathrm{out}}^{(d)}
      +\sigma_{\mathrm{out}}^{(d)}W\right), \\
    f_{\theta_t}(x_d)[y_d]\mid H_1
    &=h\!\left(\mu_{\mathrm{out}}^{(d)}
      +\sigma_{\mathrm{out}}^{(d)}(\eta^{(d)}+\rho W)\right),
\end{align*}
where $W\sim\mathcal{N}(0,1)$. By definition,
$\xi^{(d)}=[1-s^{(d)}]\sigma_{\mathrm{out}}^{(d)}$.

We first expand the normalizing term $\bar c^{(d)}$.
Since the reference models are balanced between the IN and OUT populations,
$\bar c^{(d)}$ is the average of the two conditional expectations above.
The first-order expansion
\begin{align*}
    h\!\left(\mu_{\mathrm{out}}^{(d)}+u\right)
    &\approx
    s^{(d)}+s^{(d)}(1-s^{(d)})u
\end{align*}
gives
\begin{align*}
    \mathbb{E}\!\left[
      h\!\left(
        \mu_{\mathrm{out}}^{(d)}
        +\sigma_{\mathrm{out}}^{(d)}W
      \right)
    \right]
    &\approx
    s^{(d)}
    +s^{(d)}(1-s^{(d)})
     \sigma_{\mathrm{out}}^{(d)}\mathbb{E}[W] \\
    &=s^{(d)},
\end{align*}
and
\begin{align*}
    &\mathbb{E}\!\left[
      h\!\left(
        \mu_{\mathrm{out}}^{(d)}
        +\sigma_{\mathrm{out}}^{(d)}
         (\eta^{(d)}+\rho W)
      \right)
    \right] \\
    &\qquad\approx
    s^{(d)}
    +s^{(d)}(1-s^{(d)})
     \sigma_{\mathrm{out}}^{(d)}
     \mathbb{E}\!\left[\eta^{(d)}+\rho W\right] \\
    &\qquad=
    s^{(d)}
    +s^{(d)}(1-s^{(d)})
     \eta^{(d)}\sigma_{\mathrm{out}}^{(d)},
\end{align*}
where we use $\mathbb{E}[W]=0$. Therefore,
\begin{align}
    \bar c^{(d)}
    &=\frac{1}{2}\Bigl\{
      \mathbb{E}\bigl[h(\mu_{\mathrm{out}}^{(d)}
      +\sigma_{\mathrm{out}}^{(d)}W)\bigr]
      +\mathbb{E}\bigl[h(\mu_{\mathrm{out}}^{(d)}
      +\sigma_{\mathrm{out}}^{(d)}(\eta^{(d)}+\rho W))\bigr]
      \Bigr\} \nonumber\\
    &\approx\frac{1}{2}\Bigl\{
      s^{(d)}+s^{(d)}
      +s^{(d)}(1-s^{(d)})
       \eta^{(d)}\sigma_{\mathrm{out}}^{(d)}
      \Bigr\} \nonumber\\
    &=s^{(d)}\left(1+\frac{\eta^{(d)}\xi^{(d)}}{2}\right),
      \nonumber
\end{align}
Using $\log(1+u)\approx u$ then gives
\begin{align}
    \log\bar c^{(d)}
    &\approx
      \log s^{(d)}
      +\log\!\left(
        1+\frac{\eta^{(d)}\xi^{(d)}}{2}
      \right) \nonumber\\
    &\approx
      \log s^{(d)}
      +\frac{\eta^{(d)}\xi^{(d)}}{2}.
    \label{eq:blr_normalizer_approx}
\end{align}

\textit{Member case.}
Since $g(\mu_{\mathrm{out}}^{(d)})=\log s^{(d)}$ and
$g'(\mu_{\mathrm{out}}^{(d)})=1-s^{(d)}$, under $H_1$,
\begin{align}
    \log f_{\theta_t}(x_d)[y_d]
    &=g\!\left(\mu_{\mathrm{out}}^{(d)}
      +\sigma_{\mathrm{out}}^{(d)}(\eta^{(d)}+\rho W)\right)
      \nonumber\\
    &\approx\log s^{(d)}
      +(1-s^{(d)})\sigma_{\mathrm{out}}^{(d)}
       (\eta^{(d)}+\rho W) \nonumber\\
    &=\log s^{(d)}+\eta^{(d)}\xi^{(d)}
      +\rho\xi^{(d)}W.
    \label{eq:blr_member_tlc_approx}
\end{align}
Substituting \eqref{eq:blr_normalizer_approx} into the BLR statistic yields
\begin{align}
    \varphi_{\mathrm{BLR}}^{(d)}\mid H_1
    &\approx\frac{\eta^{(d)}\xi^{(d)}}{2}
      +\rho\xi^{(d)}W.
    \label{eq:blr_h1_form}
\end{align}

\textit{Non-member case.}
Under $H_0$, the same first-order expansion gives
\begin{align}
    \log f_{\theta_t}(x_d)[y_d]
    &=g\!\left(\mu_{\mathrm{out}}^{(d)}
      +\sigma_{\mathrm{out}}^{(d)}W\right) \nonumber\\
    &\approx\log s^{(d)}
      +(1-s^{(d)})\sigma_{\mathrm{out}}^{(d)}W \nonumber\\
    &=\log s^{(d)}+\xi^{(d)}W.
    \label{eq:blr_nonmember_tlc_approx}
\end{align}
Combining this expression with \eqref{eq:blr_normalizer_approx} yields
\begin{align}
    \varphi_{\mathrm{BLR}}^{(d)}\mid H_0
    &\approx-\frac{\eta^{(d)}\xi^{(d)}}{2}
      +\xi^{(d)}W.
    \label{eq:blr_h0_form}
\end{align}

Since $\mathbb{E}[W]=0$, the first-order mean separation of BLR is
\begin{align}
    \Delta\mu_{\mathrm{BLR}}^{(d)}
    &=\mathbb{E}[\varphi_{\mathrm{BLR}}^{(d)}\mid H_1]
      -\mathbb{E}[\varphi_{\mathrm{BLR}}^{(d)}\mid H_0]
      \approx\eta^{(d)}\xi^{(d)}.
\end{align}
Thus, the contraction ratio $\rho$ affects the first-order variance term of the BLR statistic, but not its first-order mean separation. This confirms that BLR does not use variance contraction as a mean-separation signal.
\end{proof}

\subsection{Proof of Lemma~4.3}
\label{proof:lem:pairwise}

\begin{proof}
At $\rho=1$, the GLR forms derived above are Gaussian. For fixed points $q$ and $z$, let $W_q,W_z\overset{\mathrm{ind}}{\sim}\mathcal{N}(0,1)$.
For independent random variables $X$ and $Y$,
\begin{align}
    \mathbb{E}[X - Y] = \mathbb{E}[X] - \mathbb{E}[Y], \quad \mathrm{Var}[X - Y] = \mathrm{Var}[X] + \mathrm{Var}[Y].
\end{align}

\textit{GLR-induced difference.}
From Lemma~4.1, conditional on $\eta^{(q)}$ and $\eta^{(z)}$,
\begin{itemize}
    \item a member query satisfies
    $\varphi_{\mathrm{GLR}}^{(q)} \sim \mathcal{N}\left(\frac{(\eta^{(q)})^2}{2}, (\eta^{(q)})^2\right)$;
    \item a non-member population point satisfies
    $\varphi_{\mathrm{GLR}}^{(z)} \sim \mathcal{N}\left(-\frac{(\eta^{(z)})^2}{2}, (\eta^{(z)})^2\right)$.
\end{itemize}
Therefore, for a member query,
\begin{align}
    \varphi_{\mathrm{GLR}}^{(q)}-\varphi_{\mathrm{GLR}}^{(z)}\mid H_1
    &\sim\mathcal{N}\!\left(
    \frac{(\eta^{(q)})^2+(\eta^{(z)})^2}{2},
    (\eta^{(q)})^2+(\eta^{(z)})^2
    \right).
    \label{eq:pairwise_glr_h1}
\end{align}

\textit{BLR-induced difference.}
From Lemma~4.2, conditional on $(\eta^{(q)},\xi^{(q)})$ and $(\eta^{(z)},\xi^{(z)})$,
\begin{itemize}
    \item a member query satisfies
    $\varphi_{\mathrm{BLR}}^{(q)} \sim \mathcal{N}\left(\frac{\eta^{(q)}\xi^{(q)}}{2}, (\xi^{(q)})^2\right)$;
    \item a non-member population point satisfies
    $\varphi_{\mathrm{BLR}}^{(z)} \sim \mathcal{N}\left(-\frac{\eta^{(z)}\xi^{(z)}}{2}, (\xi^{(z)})^2\right)$.
\end{itemize}
Therefore, for a member query,
\begin{align}
    \varphi_{\mathrm{BLR}}^{(q)}-\varphi_{\mathrm{BLR}}^{(z)}\mid H_1
    &\sim\mathcal{N}\!\left(
    \frac{\eta^{(q)}\xi^{(q)}+\eta^{(z)}\xi^{(z)}}{2},
    (\xi^{(q)})^2+(\xi^{(z)})^2
    \right).
    \label{eq:pairwise_blr_h1}
\end{align}

We finally consider a non-member query. Under $H_0$, both $q$ and $z$ follow their respective non-member distributions, giving
\begin{align}
    \varphi_{\mathrm{GLR}}^{(q)}-\varphi_{\mathrm{GLR}}^{(z)}\mid H_0
    &\sim\mathcal{N}\!\left(
    \frac{(\eta^{(z)})^2-(\eta^{(q)})^2}{2},
    (\eta^{(q)})^2+(\eta^{(z)})^2
    \right),
    \label{eq:pairwise_glr_h0}\\
    \varphi_{\mathrm{BLR}}^{(q)}-\varphi_{\mathrm{BLR}}^{(z)}\mid H_0
    &\sim\mathcal{N}\!\left(
    \frac{\eta^{(z)}\xi^{(z)}-\eta^{(q)}\xi^{(q)}}{2},
    (\xi^{(q)})^2+(\xi^{(z)})^2
    \right).
    \label{eq:pairwise_blr_h0}
\end{align}

Although the conditional means need not be zero, the marginal null distributions are symmetric about zero. Indeed, exchanging the i.i.d. non-member draws $Q$ and $Z$ reverses the sign of the pairwise difference without changing its distribution. This symmetry leads directly to the neutral baselines in Lemma~4.4.
\end{proof}

\subsection{Proof of Lemma~4.4}

\begin{proof}
We establish the two neutral baselines in turn.

\textit{Binary-voting population calibration.}
For binary-voting population calibration, we use the neutral threshold $\gamma=0$. Let $Q$ and $Z$ be independent non-member points drawn from the same data distribution. Under $H_0$, the pair $(Q,Z)$ is exchangeable, and hence
\begin{align}
    \Delta(Q,Z) &\overset{d}{=} \Delta(Z,Q)=-\Delta(Q,Z).
\end{align}
Since the pairwise difference is continuous, ties have probability zero. Therefore,
\begin{align}
    \Pr\left[\Delta(Q,Z)>0\mid H_0\right]
    &= \Pr\left[\Delta(Q,Z)<0\mid H_0\right]
    = \frac{1}{2}. \label{eq:proof_binary_win_prob}
\end{align}
The binary score is the query point's population win probability. Taking the expectation over a random non-member query gives
\begin{align}
    \mathbb{E}\left[\mathrm{Score}_{\mathrm{Binary}}(Q)\mid H_0\right]
    &= \frac{1}{2},
\end{align}
which establishes the binary-voting baseline in Lemma~4.4 of the main manuscript.

\textit{Cauchy-combination population calibration.}
Let $P=p(Q,Z)$ be an exact continuous pairwise null $p$-value. Under $H_0$, $P\sim\mathrm{Uniform}(0,1)$. Define the Cauchy evidence transformation~\citep{cct}
\begin{align}
    T(P) &= \tan((0.5-P)\pi).
\end{align}
For any $t\in\mathbb{R}$,
\begin{align}
    \Pr\left[T(P)\leq t\mid H_0\right]
    &= \Pr\left[P\geq \frac{1}{2}-\frac{1}{\pi}\arctan(t)\mid H_0\right] \nonumber\\
    &= \frac{1}{2}+\frac{1}{\pi}\arctan(t). \label{eq:proof_cauchy_cdf}
\end{align}
This is the CDF of a standard Cauchy random variable. Hence, each null pairwise $p$-value is mapped to Cauchy evidence with median zero and undefined mean:
\begin{align}
    \mathrm{Median}\left(T(p(Q,Z))\mid H_0\right)=0. \label{eq:proof_cauchy_pairwise_median}
\end{align}
The Cauchy-combined \plmia score averages these transformed pairwise evidence terms. Under independence, Cauchy stability gives an exact median-zero Cauchy score. In the pairwise setting, the terms share the same query point; the score-level approximation used in the main manuscript is therefore
\begin{align}
    \mathrm{Median}\left(\mathrm{Score}_{\mathrm{PL}}(Q)\mid H_0\right)
    &\approx 0. \label{eq:proof_score_pl_median}
\end{align}
This establishes the Cauchy-combination baseline in Lemma~4.4 of the main manuscript. The median is used in place of the expectation because Cauchy evidence has no finite mean.
\end{proof}

\subsection{Proof of Theorem~4.5}

\begin{proof}
We first derive the Gaussian-CDF form for \lira, \rmia, and \plmiap.
For each attack $A$, let $S_A$ denote the corresponding pointwise statistic before binary population calibration. Binary population calibration applies an increasing null-CDF transformation to this statistic and therefore preserves its rejection region and attack power. It is thus sufficient to derive the Gaussian power expression using $S_A$.
At $\rho=1$, for $A\in\{\mathrm{LiRA},\mathrm{RMIA},\mathrm{PL\text{-}MIA}^{+}\}$, write the null and member score distributions as
\begin{align}
    S_A\mid H_0 &\sim \mathcal{N}(\mu_0^A,(\sigma_A)^2), \\
    S_A\mid H_1 &\sim \mathcal{N}(\mu_1^A,(\sigma_A)^2).
\end{align}
At fixed FPR $\alpha$, the null threshold is
\begin{align}
    \tau_\alpha^A
    &= \mu_0^A+\sigma_A\Phi^{-1}(1-\alpha). \label{eq:proof_gaussian_threshold}
\end{align}
Therefore,
\begin{align}
    \mathrm{TPR}_{A}(\alpha)
    &= \Pr\left[S_A>\tau_\alpha^A\mid H_1\right] \nonumber\\
    &= 1-\Phi\left(\frac{\tau_\alpha^A-\mu_1^A}{\sigma_A}\right) \nonumber\\
    &= \Phi\left(\frac{\mu_1^A-\mu_0^A}{\sigma_A}-\Phi^{-1}(1-\alpha)\right). \label{eq:proof_gaussian_tpr}
\end{align}
This is the standard fixed-level Gaussian power calculation used in likelihood-ratio testing and LiRA-style MIA analyses~\citep{neyman1933ix}.
By the rejection-region preservation above, the same expression gives the attack power of the corresponding binary-calibrated attack. Defining
\begin{align}
    \kappa_A
    &:= \frac{\mu_1^A-\mu_0^A}{\sigma_A}
\end{align}
gives the Gaussian attack-power expression in Theorem~4.5 of the main manuscript. Since $\Phi(\cdot)$ is strictly increasing, a larger $\kappa_A$ gives a higher TPR at the same FPR.

We next derive the low-FPR expression for the Cauchy-combined \plmia score. The score is
\begin{align}
    \mathrm{Score}_{\mathrm{PL}}(Q)
    &= \frac{1}{|Z|}\sum_{z\in Z}T(p(Q,z)),
    \qquad T(p)=\tan((0.5-p)\pi). \label{eq:proof_pl_score_cauchy}
\end{align}
The analytic Cauchy upper-tail threshold at FPR $\alpha$ is
\begin{align}
    \tau_\alpha
    &=\tan(\pi(0.5-\alpha)). \label{eq:proof_cauchy_threshold}
\end{align}
\textbf{In the low-FPR regime, the upper tail of the Cauchy average is dominated by one or a few very large transformed terms~\citep{cct}.}
A transformed term reaches the scale $|Z|\tau_\alpha$ when
\begin{align}
    T(p(Q,z))\geq |Z|\tau_\alpha.
\end{align}
Since $T(p)$ is strictly decreasing in $p$, this corresponds to
\begin{align}
    p(Q,z)
    &\leq \frac{1}{2}-\frac{1}{\pi}\arctan\left(|Z|\tau_\alpha\right)
    := u_{\alpha,|Z|}. \label{eq:proof_u_alpha}
\end{align}
For small $\alpha$, $u_{\alpha,|Z|}\approx \alpha/|Z|$. Let
\begin{align}
    r_\alpha
    &= \Pr\left[p(Q,Z)\leq u_{\alpha,|Z|}\mid H_1\right]. \label{eq:proof_r_alpha}
\end{align}
Here, $r_\alpha$ is the probability that one pairwise comparison from a member query produces such a very small $p$-value. Approximating the upper-tail event of the Cauchy average by the event that at least one population comparison produces such evidence gives
\begin{align}
    \mathrm{TPR}_{\mathrm{PL\text{-}MIA}}(\alpha)
    &\approx \Pr\left[\exists z\in Z: p(Q,z)\leq u_{\alpha,|Z|}\mid H_1\right]. \label{eq:proof_cauchy_tail_event}
\end{align}
Applying the product approximation over the $|Z|$ population comparisons yields
\begin{align}
    \mathrm{TPR}_{\mathrm{PL\text{-}MIA}}(\alpha)
    &\approx 1-(1-r_\alpha)^{|Z|}
    \label{eq:proof_cauchy_tpr}
\end{align}
\eqref{eq:proof_cauchy_tpr} establishes the Cauchy-combination attack-power expression in Theorem~4.5 of the main manuscript. It characterizes the low-FPR regime in which one or a few very small pairwise $p$-values drive the upper tail of the Cauchy average. Lemma~4.3 explains why member queries are more likely to produce such evidence: larger pairwise differences correspond to smaller $p$-values.
\end{proof}

\subsection{Proof of Theorem~4.6}

\begin{proof}
We prove the three comparisons in turn.

\textbf{Part (I): $\mathrm{PL\text{-}MIA}^{+}$ versus \lira.}
Let $G(q):=\varphi_{\mathrm{GLR}}^{(q)}$, and let $F_{0,G}(t):=\Pr[G(Z)\leq t\mid H_0]$ denote the marginal null CDF of the GLR score, where $Z$ is an independent non-member population draw.
Under $\rho=1$, Lemma~4.1 shows that the null GLR score is conditionally Gaussian for every point. Hence, $F_{0,G}$ is continuous and strictly increasing.
With population calibration, the binary score is the probability that the query GLR score exceeds that of a non-member population point. Conditional on $G(q)$, the query score is fixed, and therefore
\begin{align}
    \mathrm{Score}_{\mathrm{Binary}}(q) =\Pr_Z\!\left[G(Z)<G(q)\mid G(q)\right]  =F_{0,G}\!\left(G(q)\right).
    \label{eq:proof_score_transform}
\end{align}
Thus, population calibration maps the raw GLR score of $q$ to \textbf{its percentile under the non-member distribution.}

For a non-member query $Q$, the random variable $G(Q)$ follows the same null CDF $F_{0,G}$. The probability-integral transform therefore gives
\begin{align}
    \mathrm{Score}_{\mathrm{Binary}}(Q)\mid H_0 \sim\mathrm{Uniform}(0,1).
\end{align}
Consequently, the size-$\alpha$ rejection region of $\mathrm{PL\text{-}MIA}^{+}$ is
\begin{align}
\left\{F_{0,G}(G(Q))>1-\alpha\right\} = \left\{G(Q)>F_{0,G}^{-1}(1-\alpha)\right\}.
\label{eq:proof_rejection_equivalence}
\end{align}
The event on the right is precisely the size-$\alpha$ rejection event obtained by thresholding the original GLR score, as in \lira. Hence, the two attacks have the same rejection event and therefore the same attack power:
\begin{align}
\mathrm{TPR}_{\mathrm{PL\text{-}MIA}^{+}}(\alpha) =\mathrm{TPR}_{\mathrm{LiRA}}(\alpha).
\label{eq:proof_equal_lira}
\end{align}
When $\sigma_\eta^2>0$, the raw GLR distributions vary across query
points. \eqref{eq:proof_score_transform} expresses these heterogeneous scores on the common null-percentile scale $[0,1]$, while preserving their ordering and attack power. This establishes Part (I).

\textbf{Part (II): $\mathrm{PL\text{-}MIA}^{+}$ versus RMIA.}
Both attacks use population calibration, so their difference lies in the pointwise statistic: $\mathrm{PL\text{-}MIA}^{+}$ uses GLR, whereas RMIA uses BLR.
Let
\begin{equation}
    G:=\varphi_{\mathrm{GLR}}^{(Q)}, \qquad B:=\varphi_{\mathrm{BLR}}^{(Q)}.
\end{equation}
For either pointwise statistic $S\in\{G,B\}$, binary population calibration maps $S$ to its null percentile $F_{0,S}(S)$. Since this transformation is increasing, it leaves the rejection region induced by $S$ unchanged. It is therefore sufficient to compare GLR and BLR before calibration.

\textit{Variance contraction: $\rho<1$.}
By Equation~(5) and Lemma~4.1 of the main manuscript, GLR is the likelihood-ratio statistic for the observable confidence score.
Under $H_0$, expanding its
expression gives
\begin{align}
    G(W)
    &=
    -\log\rho
    +\frac{W^2}{2}
    -\frac{(W-\eta^{(Q)})^2}{2\rho^2}
    \nonumber\\
    &=
    \frac{1-\rho^{-2}}{2}W^2
    +\frac{\eta^{(Q)}}{\rho^2}W
    -\log\rho
    -\frac{(\eta^{(Q)})^2}{2\rho^2}.
    \label{eq:proof_glr_quadratic}
\end{align}
Its coefficient on $W^2$ is $(1-\rho^{-2})/2<0$. Hence, $G(W)$ is a concave quadratic function, and a nontrivial upper GLR rejection region therefore has the form
$\{W:G(W)>c_G\}=(w_-,w_+)$, which is a bounded interval.

In contrast, the exact BLR statistic is
\begin{align}
    B(W) &=
    \log\operatorname{sigmoid}\!\left(
        \mu_{\mathrm{out}}^{(Q)}
        +\sigma_{\mathrm{out}}^{(Q)}W \right)
    -\log\bar c^{(Q)}.
\end{align}
For a fixed query, $\bar c^{(Q)}$ is constant, and
\begin{align}
    \frac{\mathrm{d}B(W)}{\mathrm{d}W} &=
    \sigma_{\mathrm{out}}^{(Q)}
    \left[
        1-\operatorname{sigmoid}\!\left(
            \mu_{\mathrm{out}}^{(Q)}
            +\sigma_{\mathrm{out}}^{(Q)}W
        \right)
    \right]
    >0.
    \label{eq:proof_blr_monotonicity}
\end{align}
Thus, BLR is strictly increasing in $W$, and its upper rejection
region has the form $\{W:B(W)>c_B\}=(w_B,\infty)$, which is a half-line.

A bounded interval and a half-line cannot coincide. Their symmetric difference therefore contains a nonempty interval, which has positive null probability because $W\sim\mathcal{N}(0,1)$ has a strictly positive density on $\mathbb{R}$.
By the Neyman-Pearson lemma~\citep{neyman1933ix}, the GLR rejection region is most powerful among size-$\alpha$ tests.
Since its null distribution is continuous and the two rejection regions differ on a set of positive null probability, the comparison is strict. Population calibration preserves both rejection regions, and hence
\begin{align}
\mathrm{TPR}_{\mathrm{PL\text{-}MIA}^{+}}(\alpha) > \mathrm{TPR}_{\mathrm{RMIA}}(\alpha), \qquad \rho<1.
\label{eq:proof_rho_ordering}
\end{align}

\textit{$\xi$-heterogeneity: $\rho=1$.}
We now fix $\eta^{(Q)}=\eta_0>0$ to isolate $\xi$-heterogeneity. Let $\bar\xi:=\mathbb{E}[\xi^{(Q)}]$ and $\operatorname{Var}(\xi^{(Q)})=\sigma_\xi^2$.
At $\rho=1$, Lemma~4.1 gives
\begin{align}
    G\mid H_1
    &=
    \eta_0W+\frac{\eta_0^2}{2},
    &
    G\mid H_0
    &=
    \eta_0W-\frac{\eta_0^2}{2}.
\end{align}
The GLR mean separation and common within-class variance are therefore
\begin{align}
    \Delta\mu_{\mathrm{GLR}}
    &=\eta_0^2,
    &
    \sigma_{\mathrm{GLR}}^2
    &=\eta_0^2.
\end{align}
Since $\eta_0>0$,
\begin{align}
    \kappa_{\mathrm{GLR}}
    &=
    \frac{\Delta\mu_{\mathrm{GLR}}}
         {\sigma_{\mathrm{GLR}}}
    =\eta_0.
    \label{eq:proof_glr_kappa}
\end{align}

Using the first-order BLR score in Lemma~4.2 gives
\begin{align}
    B\mid H_1
    &\approx
    \xi^{(Q)}
    \left(W+\frac{\eta_0}{2}\right),
    &
    B\mid H_0
    &\approx
    \xi^{(Q)}
    \left(W-\frac{\eta_0}{2}\right).
    \label{eq:proof_blr_rho_one}
\end{align}
Because $W$ is independent of $\xi^{(Q)}$, with
$\mathbb{E}[W]=0$ and $\mathbb{E}[W^2]=1$, the class means satisfy
\begin{align}
    \mathbb{E}[B\mid H_1]
    &\approx
    \frac{\eta_0\bar\xi}{2},
    &
    \mathbb{E}[B\mid H_0]
    &\approx
    -\frac{\eta_0\bar\xi}{2}.
\end{align}
Thus, the BLR mean separation is $\Delta\mu_{\mathrm{BLR}}\approx \eta_0\bar\xi$.
Moreover, $\mathbb{E}\!\left[(\xi^{(Q)})^2\right] = \bar\xi^2+\sigma_\xi^2$.
The covariance between $\xi^{(Q)}W$ and $\xi^{(Q)}$ is zero because
$W$ is independent of $\xi^{(Q)}$ and $\mathbb{E}[W]=0$. Hence, under
either hypothesis,
\begin{align}
    \operatorname{Var}(B\mid H_b)
    &\approx
    \operatorname{Var}\!\left(\xi^{(Q)}W\right)
    +\frac{\eta_0^2}{4}
     \operatorname{Var}\!\left(\xi^{(Q)}\right)
    \nonumber\\
    &=
    \mathbb{E}\!\left[(\xi^{(Q)})^2\right]
    +\frac{\eta_0^2}{4}\sigma_\xi^2
    \nonumber\\
    &=
    \bar\xi^2+
    \left(1+\frac{\eta_0^2}{4}\right)\sigma_\xi^2,
    \qquad b\in\{0,1\}.
    \label{eq:proof_blr_common_variance}
\end{align}
It follows that
\begin{align}
    \kappa_{\mathrm{BLR}}
    &=
    \frac{\eta_0\bar\xi}
    {\sqrt{
        \bar\xi^2+
        \left(1+\eta_0^2/4\right)\sigma_\xi^2
    }}.
\end{align}
Under Assumption~2.2, $\xi^{(Q)}>0$ and hence $\bar{\xi}>0$. Therefore, when $\sigma_\xi^2>0$,
\begin{align}
\kappa_{\mathrm{BLR}} &=
\frac{\eta_0\bar\xi}
{\sqrt{\bar\xi^2+
\left(1+\eta_0^2/4\right)\sigma_\xi^2}}
< \eta_0 =\kappa_{\mathrm{GLR}},
\qquad \sigma_\xi^2>0.
\label{eq:proof_kappa_ordering}
\end{align}
Since the Gaussian attack-power expression in Theorem~4.5 is strictly increasing in $\kappa_A$, this yields
\begin{align}
\mathrm{TPR}_{\mathrm{PL\text{-}MIA}^{+}}(\alpha)
&>\mathrm{TPR}_{\mathrm{RMIA}}(\alpha).
\label{eq:proof_xi_ordering}
\end{align}
Together, \eqref{eq:proof_rho_ordering} and~\eqref{eq:proof_xi_ordering} establish Part (II).

\textbf{Part (III): \plmia versus $\mathrm{PL\text{-}MIA}^{+}$.}
Recall that $r_\alpha$ is the probability that a single pairwise comparison from a member query produces an exceptionally small $p$-value. The low-FPR expression in Theorem~4.5 approximates the power of \plmia by the probability that at least one of the $|Z|$ comparisons produces such evidence:
\begin{align}
    \mathrm{TPR}_{\mathrm{PL\text{-}MIA}}(\alpha)
    &\approx 1-(1-r_\alpha)^{|Z|}. \label{eq:proof_plmia_power_part3}
\end{align}
Since $1-x\leq\exp(-x)$ for $x\in[0,1]$, we have $(1-r_\alpha)^{|Z|}\leq \exp(-|Z|r_\alpha)$ and therefore
\begin{align}
    1-(1-r_\alpha)^{|Z|}
    &\geq 1-\exp(-|Z|r_\alpha). \label{eq:proof_exp_bound}
\end{align}
Under the condition in Part (III),
\begin{align}
    1-\exp(-|Z|r_\alpha)
    &> \mathrm{TPR}_{\mathrm{PL\text{-}MIA}^{+}}(\alpha). \label{eq:proof_part3_condition}
\end{align}
Within this low-FPR approximation, combining Equations~\ref{eq:proof_plmia_power_part3}--\ref{eq:proof_part3_condition} gives
\begin{align}
    \mathrm{TPR}_{\mathrm{PL\text{-}MIA}}(\alpha)
    &> \mathrm{TPR}_{\mathrm{PL\text{-}MIA}^{+}}(\alpha). \label{eq:proof_ordering_plmia_plplus}
\end{align}
This establishes Part (III) and completes the proof.
\end{proof}


\bibliographystyle{imsart-nameyear} 
\bibliography{plmia260911}       


\end{document}